\documentclass[10pt,journal,compsoc]{IEEEtran}
\usepackage{array}
\usepackage{textcomp}
\usepackage{subfigure}
\usepackage{stmaryrd}

\usepackage{stfloats}
\usepackage{url}
\usepackage{verbatim}
\usepackage{graphicx}
\def\BibTeX{{\rm B\kern-.05em{\sc i\kern-.025em b}\kern-.08em
    T\kern-.1667em\lower.7ex\hbox{E}\kern-.125emX}}
\usepackage{balance}

\usepackage{amsmath,amssymb,amsfonts}
\usepackage{hyperref}
\hypersetup{colorlinks=true}
\def\equationautorefname~#1\null{%
  Equation~(#1)\null
}
\usepackage{algorithm}
\usepackage[noend]{algpseudocode}
\usepackage{colortbl}
\usepackage{xcolor}
\usepackage{colortbl}%

\usepackage{capt-of,lipsum}
\usepackage{amsthm}
\definecolor{mycolor}{rgb}{0.95, 0.985, 0.93}
\doublerulesepcolor{mycolor}
\usepackage{xspace}
\usepackage{color}
\usepackage{tcolorbox}
\definecolor{mGray1}{rgb}{0.9,0.9,0.9}
\definecolor{mGray}{rgb}{0.5,0.5,0.5}
\definecolor{commentcolor}{rgb}{0.6,0.6,0.6}
\usepackage{enumitem}
\newtheorem{definition}{Definition}
\usepackage{pifont}
\usepackage{xcolor}
\usepackage{listings}
\newcommand{\m}{\mathit} 
\newcommand{\mtl}{\phi} 
\newcommand{\interval}{I} 
\newcommand{\relation}{R} 
\newcommand{\Obj}{o} 
\newcommand{\Subj}{s} 
\newcommand{\OBJ}{O} 
\newcommand{\SUBJ}{S} 
\newcommand{\Prolog}{\mathcal{P}}
\newcommand{\drule}{Q}

\newcommand{\hornarrow}{\,\text{:--}\,}
\newcommand{\curly}{\,\mathrel{\leadsto}\,}
\newcommand{\encoding}[3]{#1{\curly}(#2, #3)}
\newcommand{\deriveRules}[3]{#1\,{\hookrightarrow}\,(#2, #3)}
\newcommand{\nm}{\m{nm}}
\newcommand{\entity}{\m{entity}}
\newcommand{\groundTruthTriples}{\widetilde{\relation}_{\m{ground}}}
\newcommand{\derivedFacts}{\widetilde{\relation}_{\m{derived}}}

\newcommand{\llmResponse}{\m{Resp}}

\newcommand{\semantic}{\widetilde{G}}
\newcommand{\similarity}{S}

\usepackage{makecell} 
\usepackage[bottom]{footmisc}
\newcommand{\RNeg}{[\relation\text{-}\m{Neg}]}
\newcommand{\RSym}{[\relation\text{-}\m{Sym}]}
\newcommand{\RInv}{[\relation\text{-}\m{Inverse}]}
\newcommand{\RTrans}{[\relation\text{-}\m{Trans}]}

\newcommand{\commentstyle}[1]{\textcolor{mGray}{\footnotesize{#1}}}

\newcommand{\entityCat}{$\m{EC}$}
\newcommand{\relationCat}{$\m{RC}$}

\algrenewcommand\algorithmicindent{1.3em}%

\usepackage[normalem]{ulem}
\newcommand{\shortNeg}{!} 
\newcommand{\syh}[1]{{\small \ttfamily \color{purple}{{{YH:#1}}}}}
\newcommand\figref[1]{Fig.\,{\ref{#1}}}
\newcommand\tabref[1]{Table \textcolor{blue}{\ref{#1}}}
\newcommand\secref[1]{\S \textcolor{blue}{\ref{#1}}}
\usepackage{adjustbox}
\usepackage{wrapfig}
\newcommand\theoref[1]{Theorem~\textcolor{blue}{\ref{#1}}}

\newcommand\defref[1]{Definition~\textcolor{blue}{\ref{#1}}}
\newcommand\algoref[1]{Algorithm~\textcolor{blue}{\ref{#1}}}
\newcommand{\trans}{\m{R}}
\newcommand{\Istart}{n_{\m{start}}}
\newcommand{\Iend}{n_{\m{end}}}
\newcommand{\history}{\mathcal{H}}
\newcommand{\timepoint}{t}
\newcommand{\distance}{d}
\newcommand\nmNEW{\nm_{\m{new}}}
\usepackage{thmtools} 
\usepackage{thm-restate}

\newcommand{\plus}{\texttt{+}}

\usepackage{subcaption}
\usepackage{booktabs}

\definecolor{keywordcolor}{rgb}{0.13,0.29,0.53}
\definecolor{stringcolor}{rgb}{0.31,0.60,0.02}
\definecolor{commentcolor}{rgb}{0.56,0.35,0.01}
\definecolor{backcolour}{rgb}{0.95,0.95,0.92}

\newcommand{\head}[1]{{\noindent\textbf{#1}}}

\newcommand{\tool}{\textsc{Drowzee}\xspace}
\newcommand{\mtltoNL}{\textsc{mtl2NL}\xspace}
\newcommand{\iyear}{\m{t}}

\newcommand{\instruction}{\textsc{Instruction}\xspace}
\newcommand{\query}{\textsc{Query}\xspace}
\newcommand{\hallucinationAnswer}{\textsc{Hallucinations Answer}\xspace}

\begin{document}


\title{Detecting LLM Fact-conflicting Hallucinations Enhanced by Temporal-logic-based Reasoning}

\author{
Ningke Li$^{\star}$,
Yahui Song$^{\star}$,
Kailong Wang$^{\dagger}$,
Yuekang Li,
Ling Shi,
Yi Liu,
Haoyu Wang
    \thanks{N. Li, K. Wang, and H. Wang are with Huazhong University of Science and Technology, China. 
    E-mail: \{lnk\_01, wangkl, haoyuwang\}@hust.edu.cn}
    \thanks{Y. Song is with the National University of Singapore, Singapore.\protect\\
    E-mail: yahui\_s@nus.edu.sg}
    \thanks{Y. Li is with the University of New South Wales, Australia.\protect\\
    E-mail: yuekang.li@unsw.edu.au}
    \thanks{L. Shi and Y. Liu are with Nanyang Technological University, Singapore.\protect\\
    E-mail: ling.shi@ntu.edu.sg, yi009@e.ntu.edu.sg}
\thanks{${\star}$ Ningke Li and Yahui Song contribute equally to this work.}
\thanks{${\dagger}$ Kailong Wang is the corresponding author.}
}


\IEEEtitleabstractindextext{%
\begin{abstract}
Large language models (LLMs) face the challenge of hallucinations -- outputs that seem coherent but are actually incorrect. A particularly damaging type is fact-conflicting hallucination (FCH), where generated content contradicts established facts. Addressing FCH presents three main challenges: \textbf{1)} Automatically constructing and maintaining large-scale benchmark datasets is difficult and resource-intensive; \textbf{2)} Generating complex and efficient test cases that the LLM has not been trained on -- especially those involving intricate temporal features -- is challenging, yet crucial for eliciting hallucinations; and \textbf{3)} Validating the reasoning behind LLM outputs is inherently difficult, particularly with complex logical relationships, as it requires transparency in the model's decision-making process. 

This paper presents \tool{}, an innovative end-to-end metamorphic testing framework that utilizes temporal logic to identify fact-conflicting hallucinations (FCH) in large language models (LLMs). \tool{} builds a comprehensive factual knowledge base by crawling sources like Wikipedia and uses automated temporal-logic reasoning to convert this knowledge into a large, extensible set of test cases with ground truth answers. LLMs are tested using these cases through template-based prompts, which require them to generate both answers and reasoning steps. To validate the reasoning, we propose two semantic-aware oracles that compare the semantic structure of LLM outputs to the ground truths. 
Across nine LLMs in nine different knowledge domains, experimental results show that \tool{} effectively identifies rates of non-temporal-related hallucinations ranging from 24.7\% to 59.8\%, and rates of temporal-related hallucinations ranging from 16.7\% to 39.2\%.
Key insights reveal that LLMs struggle with out-of-distribution knowledge and logical reasoning. These findings highlight the importance of continued efforts to detect and mitigate hallucinations in LLMs.


\end{abstract}

\begin{IEEEkeywords}
Large Language Model, Hallucination, Temporal Logic, Metamorphic Testing
\end{IEEEkeywords}
}

\maketitle


\IEEEpeerreviewmaketitle
\section{Introduction}
Large Language Models (LLMs) have revolutionized language processing, demonstrating impressive text generation and comprehension capabilities with diverse applications. However, despite their growing use, LLMs face significant security and privacy challenges~\cite{siddiq2023generate, hou2023large, kaddour2023challenges, li2024model, 10.1145/3691620.3695510}, which affect their overall effectiveness and reliability. A critical issue is the phenomenon of \emph{hallucination}, where LLMs generate outputs that are coherent but factually incorrect or irrelevant. This tendency to produce misleading information compromises the safety and usability of LLM-based systems. This paper focuses on \emph{fact-conflicting hallucina}tion (FCH), the most prominent form of hallucination in LLMs. FCH occurs when LLMs generate content that directly contradicts established facts. For instance, as illustrated in \figref{fig:example1}, an LLM incorrectly asserts that ``\emph{Haruki Murakami won the Nobel Prize in Literature in 2016}'', whereas the fact is that ``\emph{Haruki Murakami has not won the Nobel Prize, though he has received numerous other literary awards}''. 
Such inaccuracies can significantly lead to user confusion and undermine the trust and reliability that are crucial for LLM applications.

\begin{figure}[t]
\centering
\vspace{3mm}
\hspace{-3mm}
\includegraphics[width=\linewidth]{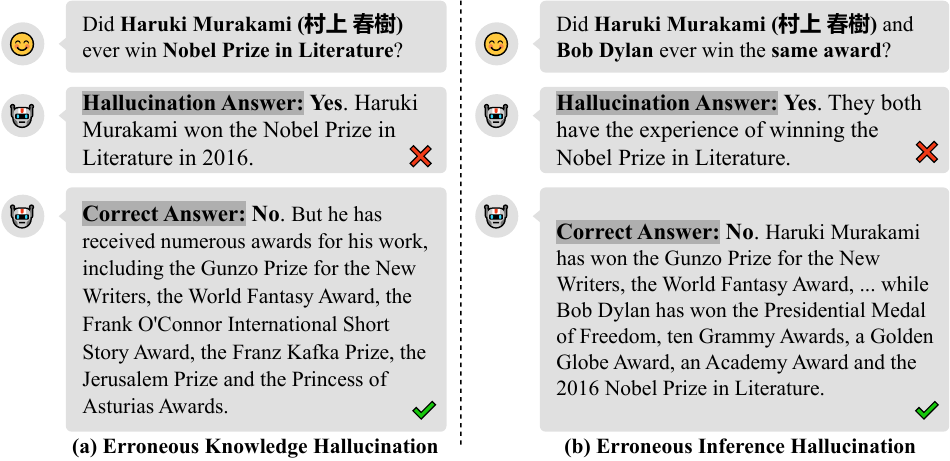}
\\[0.5em]
\caption{A Hallucination Output Example}
\label{fig:example1}
\vspace{-4mm}
\end{figure}

Recent studies have introduced various methods to detect LLM hallucinations. A common approach involves developing specialized benchmarks, such as TruthfulQA~\cite{lin-etal-2022-truthfulqa}, HaluEval~\cite{HaluEval}, and KoLA~\cite{yu2023kola}, to assess hallucinations in tasks like question-answering, summarization, and knowledge graphs. 
While manually labeled datasets provide valuable insights, current methods often rely on simplistic or semi-automated techniques such as string matching, manual validation, or verification through another language model. These approaches reveal significant gaps in automatically and effectively detecting fact-conflicting hallucinations (FCH). 
The primary challenges in FCH detection arise from the lack of dedicated ground truth datasets, the absence of comprehensive test cases designed to trigger FCH, and the lack of a robust testing framework.  
Unlike other types of hallucinations, such as input-conflicting or context-conflicting hallucinations~\cite{ji-etal-2023-rho, shi2023large}, which can often be identified through semantic consistency checks, detecting FCH requires the verification of factual accuracy against external knowledge sources/databases. This process is particularly challenging and resource-intensive, especially for tasks that involve complex logical relationships~\cite{zhang2024fusion}. We identify three primary challenges in addressing this research gap:

\begin{enumerate}[itemsep=1mm, wide,  labelindent=9pt]
\item {\textbf{Automatically constructing and updating benchmark datasets.}} Existing methodologies mainly rely on manually curated benchmarks for detecting specific hallucinations, which fail to encompass the broad and dynamic spectrum of fact-conflicting scenarios in LLMs. 
Meanwhile, due to the ever-evolving nature of knowledge, the need for frequent updates to benchmark data imposes a substantial and continuous maintenance effort.
The reliance on benchmark datasets thus restricts the FCH detection techniques' adaptability, scalability, and  
detection capability;  
\item {\textbf{Efficiently generating FCH test cases.}}
LLMs often answer correctly to simple, straightforward questions due to their extensive training on vast datasets. However, to effectively assess their reasoning capabilities, it is important to generate more complex questions, such as those involving intricate temporal characteristics, that require reasoning rather than just recalling facts. However, constructing such test cases is non-trivial. The challenge lies in designing questions that use familiar knowledge but involve reasoning patterns the LLM may not have been explicitly trained on. Creating such test cases efficiently while ensuring they probe reasoning skills in ways the model has not previously encountered is essential to uncovering latent hallucinations;
\item {\textbf{Validating the reasoning steps from LLM outputs.}} Even when LLMs finally produce correct answers, the outputs may not indicate an accurate reasoning process, potentially masking false understanding -- a source of FCH. Additionally, the quality of manual validation can differ based on human expertise. As a result, automatically validating reasoning processes, particularly those involving complex logical relationships, is inherently challenging. 
\vspace{1mm}
\end{enumerate}

To address the problems outlined above, this paper presents a novel automatic end-to-end metamorphic testing technique based on temporal logic for detecting FCH. To the best of our knowledge, we are the first to create a comprehensive FCH testing framework that utilizes factual knowledge reasoning and metamorphic testing, all seamlessly integrated into the fully automated tool, \tool. 

\tool begins by establishing a comprehensive factual knowledge base sourced through crawling information from accessible knowledge bases such as Wikipedia. Each piece of this knowledge acts as a ``seed'' for subsequent transformations. Leveraging logical operators to automatically generate temporal reasoning rules, we transform and augment these seeds and expand factual knowledge into a well-established set of question-answer pairs.
Using the questions and answers in the knowledge set as test cases and ground truth, respectively, we construct a reliable and robust FCH testing benchmark.

The experiment uses a series of carefully designed template-based prompts to test for FCHs in LLMs. To thoroughly evaluate the reasoning behind the responses, we instruct the LLMs not only to generate answers to the test cases but also to provide detailed justifications for their answers. To reliably identify FCH, we introduce two semantic-aware, similarity-based metamorphic oracles. These oracles extract the key semantic elements from each sentence and map out the logical relationships between them. By comparing the logical and semantic structures of the LLM's responses with the ground truth, the oracles can detect FCH by identifying significant deviations in the LLM's answers from the correct information.



We demonstrate the effectiveness of our approach through comprehensive experiments in multiple contexts. First, our evaluation involves deploying \tool across a wide range of topics drawn from a diverse selection of Wikipedia articles. Second, we test our framework on various open-source and commercial LLMs, thoroughly assessing its applicability and performance across different model architectures. 
Our key findings indicate that \tool succeeds in automatically generating practical test cases and identifying hallucination issues of nine LLMs across nine domains. 
Using these test sets, our experiments show that the rate of hallucination responses produced by various LLMs ranges from 24.7\% to 59.8\% for cases unrelated to temporal reasoning and 16.7\% to 39.2\% for cases requiring temporal reasoning. 
We then categorize these hallucination responses into \emph{erroneous knowledge hallucination} and \emph{erroneous inference hallucination}. 
Through an in-depth analysis, we unveil that the lack of logical reasoning capabilities contributes the most to the FCH issues in LLMs. 
Additionally, we observe that LLMs are particularly prone to generating hallucinations in test cases involving temporal concepts and out-of-distribution knowledge. 
Such an evaluation demonstrates that the 
test cases generated using 
logical reasoning rules can effectively trigger and detect LLM hallucinations.  

This paper builds upon the earlier version~\cite{DBLP:journals/pacmpl/LiL0SW024} by incorporating hallucination detection through temporal-logic-guided test generation. It includes additional motivational examples (\secref{sec:motivating}), a comprehensive set of reasoning rules for encoding \emph{Metric Temporal Logic} (MTL)~\cite{DBLP:conf/lics/OuaknineW05} formulae (\secref{sec:encoding_MTL}) and automatically generating temporal-logic-related question-answer pairs (\secref{prompt}), and further experimental studies (the {RQ4} at \secref{sec:eval}) that detect hallucinations due to insufficient temporal reasoning capabilities. The main contributions of this work are summarized as follows: 
\begin{itemize}[itemsep=1mm,leftmargin=0.35cm]
\item 
\textbf{A novel FCH testing framework.} 
To the best of our knowledge, 
we are the first to develop a novel testing framework based on logic programming and metamorphic testing to automatically detect FCH issues in LLMs. 
\item \textbf{An extensive benchmark based on factual knowledge.} 
To facilitate collaborative efforts and future advances in identifying FCH, 
the source code of \tool and constructed benchmark dataset are publicly available  \cite{drowzee}. 
\item \textbf{Test generation via temporal reasoning.} 
Our tool automatically generates test cases that provide a more comprehensive evaluation of LLMs in handling reasoning tasks and identifying factual inconsistencies. By applying temporal logic-based reasoning rules, we expand the initial seed data from our knowledge base, enhancing the diversity and complexity of the test scenarios. 

\item \textbf{Semantic-aware oracles for LLM answer validation.} We propose and implement two automated verification mechanisms, i.e., the oracles, that analyze the semantic structure similarity between sentences. These oracles are designed to validate the reasoning logic behind the answers generated by LLMs, hereby reliably detecting the occurrence of FCHs. 

\end{itemize}

\section{Motivating Examples}\label{sec:motivating}

\begin{figure*}[!ht]
\centering
\includegraphics[width=0.85\linewidth]{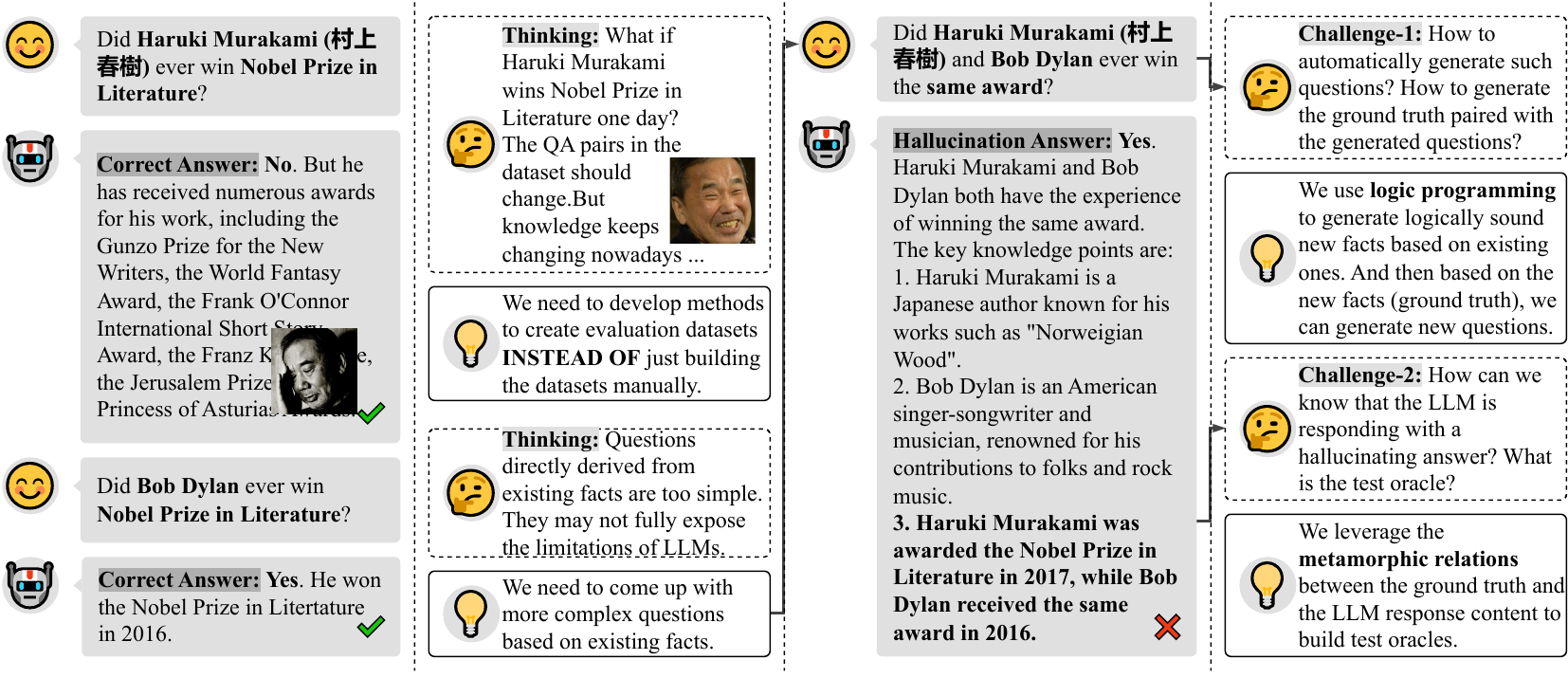}\\
\caption{Motivating Examples for Automatic Benchmark Construction with Complex Questions}
\label{fig:motivating}
\vspace{-0.1cm}
\end{figure*}



\subsection{Automatic Benchmark Construction}
As a first motivating example, shown in \figref{fig:motivating}, given 
the facts about whether Haruki Murakami and Bob Dylan have won the Nobel Prize, as illustrated in the left sub-figure, we can query straightforward questions such as ``\emph{whether Haruki Murakami or Bob Dylan has won the Nobel Prize?}''. 
Asking and verifying this knowledge requires no logical reasoning. 
However, such questions are often not enough to unveil hallucinations. 
{Therefore, more diversified questions, i.e., with intertwined and complex information, as illustrated in the right sub-figure, are needed.} 
{
Moreover, the knowledge landscape is dynamic, with new information continuously surfacing and older information becoming obsolete.} If facts change continuously over time, for instance, if Haruki Murakami were to win the Nobel Prize in the future, this would necessitate regular updates and corrections to the ground truth in existing datasets to reflect them. However, maintaining the accuracy of these benchmarks requires a significant amount of manual labor. 

\subsection{Questions Involving Temporal Reasoning}
As LLMs increasingly rely on temporal reasoning to process time-dependent data, understanding how well they can handle temporal logic is crucial for their development and deployment in real-world applications. 
Reasoning temporal-logic-related queries often requires complex steps, particularly regarding the relative timing of multiple events. Unfortunately, manually generated test cases frequently lack thorough completeness verification, which undermines their reliability. 
We demonstrate such an example in \tabref{tab:TemporalLogicRelatedHallucination}. 
The question pertains to the relative timing relationship between the original Ben 10 TV series and the year 2000. It is known that the original Ben 10 series premiered on December 27, 2005, and concluded in 2008. Therefore, when considering the query asking whether ``Ben 10 finally airs within the first three years after 2000'', the expected answer is ``No'', as none of the years from 2005 to 2008 fall within that queried timeframe. 
However, despite the accurate facts inputs, LLMs still generate hallucinations due to insufficient temporal reasoning.

\begin{table}[!b]
\centering
\renewcommand{\arraystretch}{1.1}
\begin{tabular}{l}
\toprule 
\rowcolor{mycolor} \textbf{\query:}  Did the Original Series of Ben 10 TV show finally air  \\
\rowcolor{mycolor} 
within the time frame of the 1st year to the 3rd year after the year \\
\rowcolor{mycolor}
2000? Please answer Yes or No and explain why. \\[0.5em] \hline
\\[-0.7em] 
\rowcolor{mycolor} 
\textbf{\hallucinationAnswer \rm{(Claude 3.5)}:} Yes, the original series \\
\rowcolor{mycolor}
of Ben 10 aired within the time frame of the 1st year to the 3rd    \\
\rowcolor{mycolor}
year after the year 2000. The show premiered on December 27,  \\
\rowcolor{mycolor}
2005, and concluded its run on April 15, 2008. This timeline 
\\
\rowcolor{mycolor}
clearly falls within the specified period from 2001 to 2003.
\\ 
\bottomrule 
\end{tabular}
\caption{Temporal Logic Reasoning Hallucination}
\label{tab:TemporalLogicRelatedHallucination}
\end{table}

\begin{figure}[!b]
\centering
\begin{lstlisting}[xleftmargin=3.5em,numbersep=8pt,basicstyle=\footnotesize\ttfamily] 
// Ground Facts crawled from Wikipedia
begin('Ben_10',2005). end('Ben_10',2008). 

// Generating the time-stamped fact
ben_10(Start,End) :- Start=<End, 
    begin('Ben_10',Start), end('Ben_10',End).
    
// Encding the MTL formula (*@\color{commentcolor}{$\mtl$}@*)
finally_ben_10_during_1_3(Start,End) :-
    ben_10(Start1,End1), 
    Start is (Start1-3), End is (End1-1), 
    (Start1-3)=<(End1-1), Start=<End.
    
// The time interval (*@\color{commentcolor}{$\interval$}@*) which satisfies (*@\color{commentcolor}{$\mtl$}@*)
?- finally_ben_10_during_1_3(Start,End).
   Start = 2002, End = 2007.
\end{lstlisting} 
\caption{Prolog Encoding for $\mtl \,{=}\, \mathcal{F}_{[1, 3]}(\m{Ben\_10})$}
\label{fig:prologRulesForFinally}
\end{figure}

In this work, our proposed testing framework automatically generates such temporal test cases in the form of MTL formulae, denoted by $\mtl$, which incorporate quantitative timing constraints and enable the expression of temporal relationships with precise intervals.   
For example, this query shown in \tabref{tab:TemporalLogicRelatedHallucination} is represented as  
$\mtl {\,=\,} \mathcal{F}_{[1, 3]}(\m{Ben\_10})$, where $\mathcal{F}$ stand for \emph{finally} and [1,3] is the time frame, querying the time intervals $\interval$ during which the event $\m{Ben\_10}$ finally happen within the time frame of the 1st year to the 3rd year. 
To obtain the ground truth interval $\interval$, \tool{} automatically generates the Prolog encoding rules for $\mtl$, as shown in \figref{fig:prologRulesForFinally}. These rules accurately determine that $\interval{\,=\,}[2002,2007]$, indicating that all time points within this interval satisfy $\mtl$. 
Then, the validity of 2000 is easily disproved because 2000 is not a member of $\interval$. 

It's important to note that such reasoning rules can be generated for arbitrarily nested MTL formulae. These rules lead to sound and deterministic conclusions regarding a ``Yes" or ``No" response, and they provide reasoning steps to assist in evaluating the LLMs' answers, especially when the queries become more complex. 
Furthermore, obtaining the ground truth interval directly enables us to flexibly control the generation of both positive and negative test cases. 

\section{Methodology}\label{sec:method}

\begin{figure}[!h]
\vspace{-2mm}
\centering
\includegraphics[width=1\linewidth]{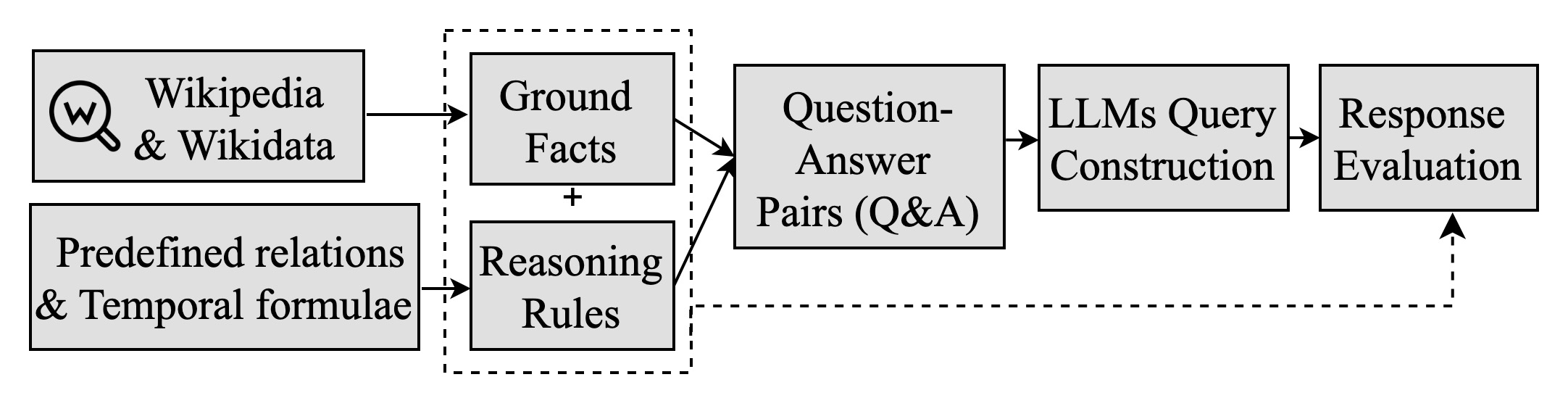}
\vspace{-5mm}
\caption{\tool Overview }
\label{fig:tool_overview}
\vspace{-1mm}
\end{figure}

\tool (\figref{fig:tool_overview}) is a general-purpose testing framework that evaluates the LLM outputs for automatically generated test cases. 
The inputs for the response evaluation
contain a natural language (NL) query for LLM and its ground truth answer obtained using logic programming (\secref{subsec3.1}).  
Based on voluminous knowledge database dumps, \tool extracts factual knowledge (\secref{knowledge}), which outputs a set of 
predicates
in the form of Prolog facts. 
Then, \tool deploys a set of pre-defined or automatically generated reasoning rules to
extend the database with a set of derived facts (\secref{sec:derive_more_facts}, \secref{sec:encoding_MTL}). 
These derived facts facilitate an automated test generation (\secref{prompt}), which outputs question-answer pairs (Q\&A) and concrete prompts for LLMs. 
Given the Q\&A pairs and the LLM outputs, \tool evaluates the responses from LLMs and detects factual inconsistency automatically (\secref{response}). 
To this end, it 
first parses LLM outputs semantic-aware structure, and evaluates their semantic similarity to the ground truth. 
Lastly, it develops similarity-based oracles that apply metamorphic testing to assess consistency against the ground truth.



\subsection{Preliminary}
\label{subsec3.1}

\begin{figure}[!b]
{
\vspace{-2mm}
\centering
\small
$
\arraycolsep=3pt\def\arraystretch{1}
\begin{array}{@{}lrcl}
\m{(Program)}&  \Prolog &{::=  } &
\widetilde{\relation} \,\plus\plus\,   \widetilde{\drule} 
\\
\m{(Rule)} &  \m{\drule} &{ ::=  } & 
\relation ~\hornarrow~ \widetilde{body}
\\[0.3em]
\m{(Body)} & \m{body} &{  ::=  } & 
{\tt{Pos}}~ \relation
\,\mid\, {\tt{Neg}}~ \relation 
\,\mid\, \pi
  \\
\m{(Predicate)} &  \relation &{  ::=  } &
 \m{\nm}\,(\widetilde{\entity}) 
\\[0.3em]
 \m{(Pure)}  &\pi &{::=}~&
{ T }
  \mid  F
 \mid  {\m{bop}(}{t_1, t_2}{)}
 \mid   {{\pi_1}}  {\wedge}  \pi_2
 \mid  {{\pi_1}} {\vee} \pi_2
 \mid  \neg\pi
\\[0.3em]
 \m{(Term)}  &t &{::=}~& c 
 \mid X 
 \mid t_1{\text{\ttfamily +}}t_2
 \mid t_1\text{-}t_2
\end{array}$
\caption{A Core Syntax of Prolog}
\label{fig:Syntax_of_Prolog}
}
\end{figure}

Logic programming allows the programmer to specify the rules and facts, enabling the Prolog interpreter to infer answers to the given queries automatically. 
We define a core syntax of Prolog in \figref{fig:Syntax_of_Prolog}. 
A Prolog program consists of two parts: a set of facts ($\widetilde{\relation}$) and a set of rules ($\widetilde{\drule}$). 
Throughout the paper, we use the over-tilde notation to denote a set of items. 
For example, $\widetilde{X}$ refers to a set of variables, i.e., $\{X_1, \dots, X_n\}$. 
A fact is represented as a relational predicate with a name and a set of entity arguments, where $\nm$ is an arbitrary distinct identifier drawn from a finite set of relations. 
Entities are drawn from the knowledge database, ranging from string types (for names or events) and integers (for time points).  
A Prolog rule is a Horn clause that comprises a head predicate and a set of body predicates placed on the left and right side of the arrow symbol ($\hornarrow$).

A rule means that the left-hand side is logically implied by the right-hand side. 
The rule bodies are either positive or negative relations, corresponding to the requirements upon the presence or absence of facts. 
We use ``$\relation$'' and ``$\shortNeg\,\relation$'' as abbreviations for
``${\tt{Pos}}~\relation$'' and ``${\tt{Neg}}~\relation$'', respectively. 
Rule bodies contain pure formulae and simplified and decidable sets of Presburger arithmetic predicates over local variables. 
The Boolean values of \emph{true} and \emph{false} are respectively indicated by $T$ and $F$. 
The binary operators $bop$ are from $\{ {<}, {\leq}, {=}, {\geq}, {>} \}$. 
Terms consist of constants (denoted by $c$), program variables (denoted by $X$ 
), or simple computations of terms, such as $t_1{\plus} t_2$ and $t_1\text{-}t_2$. 

\subsection{Factual Knowledge Extraction}
\label{knowledge} 
To facilitate an automated reasoning system, we extract the \emph{ground facts} in the structure of three-element predicates, i.e., $\m{\nm}\,(\Subj,\Obj)$, where ``$\Subj$'' (stands for $\m{subject}$) and ``$\Obj$'' (stands for $\m{object}$) are entities, and ``$\nm$'' is the name of the predicate. 
Here, we follow the convention of Prolog, where variable names must start with an uppercase letter, and any name that begins with a lowercase letter is a constant. 

Existing knowledge databases~\cite{freebase, DBpedia, Yago, WordNet} not only encompass a vast array of documents but also provide structured data, facilitating an ideal source for constructing a rich factual knowledge base. 
Thus, the genesis of our test cases is exclusively rooted in the entities and structured relations sourced from existing knowledge databases, ensuring a sophisticated and well-informed foundation for our testing framework. 
Specifically, we follow the categorization for entities (\figref{table:categories}) and relations (\figref{table:relations}) used by WikiPedia~\cite{DBpedia} to perform a thorough facts extraction. 
In particular, the {\small\textbf{Prop.}} (stands for properties) entry for relations guides the automated generation of reasoning rules detailed in \secref{sec:derive_more_facts}.


\begin{figure}[!b]
\vspace{-3mm}
\renewcommand{\arraystretch}{1.0}
\setlength{\tabcolsep}{2pt}
\footnotesize 
\begin{tabular}{l | l }
\Xhline{1.0\arrayrulewidth}
\textbf{Entity Cat.} & \textbf{Description}\\
        \Xhline{\arrayrulewidth}
        {Culture and the Arts} & Famous films, books, etc.\\ 
        {Geography and Places} & Countries, cities and locations. \\
        {Health and Fitness} & Diseases and genes. \\
        {History and Events} & Famous historical events, etc. \\
        {People and Self} & Important figures. \\
        {Mathematics and Logic} & Formulas and theorems. \\
        {Natural and Physical Sciences} & Celestial bodies and astronomy. \\
        {Society and Social Sciences} & Major social institutions, etc.\\ 
        {Technology and Applied Sciences} & Computer science, etc. \\
        \Xhline{1.5\arrayrulewidth} 
    \end{tabular}
\caption{{Entity Categorization.}}
\label{table:categories}
\end{figure}

\begin{figure}[!b]
\centering
\def\arraystretch{1.1}
\setlength{\tabcolsep}{2pt}
\footnotesize
\begin{tabular}{l | l | l}
\Xhline{1.5\arrayrulewidth}
\textbf{Relation Cat.} & \textbf{Examples}
& 
\textbf{Prop.} 
\\
\Xhline{1.5\arrayrulewidth}
        {Noun Phrase} & 
\begin{tabular}[l]{@{}l@{}} \textit{place\_of\_birth\,(barack\_obama, honolulu).}\\ \textit{genre\,(28\_days\_later, horror\_film).} \end{tabular}
        &
\begin{tabular}[l]{@{}l@{}} 
        $\RNeg$ \\
        $\RSym$ \\
        $\RTrans$
        \end{tabular}
        \\
        \hline
        \begin{tabular}[l]{@{}l@{}} Verb Phrase \\  
        (Passive Voice) \end{tabular}
        & \begin{tabular}[l]{@{}l@{}} \textit{killed\_by}\,\textit{(alexander\_pushkin}, \\ \quad  \textit{georges-charles\_de\_heeckeren\_d'anthès)}.\\ \textit{located\_in\_time\_zone\,(arizona, utc-07:00).}\\ 
        \end{tabular}
        &
\begin{tabular}[l]{@{}l@{}} 
        $\RNeg$ \\
        $\RInv$
        \end{tabular}
        \\
        \hline
        \begin{tabular}[l]{@{}l@{}} Verb Phrase \\  
        (Active Voice) \end{tabular}
        & \begin{tabular}[l]{@{}l@{}} \textit{follows\,(4769\_Castalia, 4768\_hartley).}\\ \textit{replaces\,(american\_broadcasting\_company,} \\ \qquad \quad\ \   \textit{nbc\_blue\_network).} \end{tabular}
        &
\begin{tabular}[l]{@{}l@{}}  
        $\RNeg$ \\
        $\RInv$
        \end{tabular}
        \\
        \Xhline{1.5\arrayrulewidth} 
    \end{tabular}
\caption{Relation Categorization.}
\label{table:relations}
\end{figure}

The facts extraction is done per-category basis, implementing a divide-and-conquer strategy, which efficiently integrates all the facts from all the categories. 
As shown in \algoref{alg:ground_truth}, for any given entity category and relation category, the function $\textsc{ExtractGroundFacts}$ iterates through all possible entities and relations. 
For each combination ($\m{entity}, \nm$), it queries the database using the $\textsc{QueryDB}$ function, which retrieves all three-element facts established with the specific predicate $\nm$ and the argument $\m{entity}$ placed either in the subject or the object position. 

{
\begin{algorithm}[!h]
\caption{Facts Extraction}
\label{alg:ground_truth}
\small
\begin{algorithmic}[1]
\Require  
Entity Category (\entityCat), Relation Category (\relationCat)
\Ensure Ground Facts ($\groundTruthTriples$)
\Function{ExtractGroundFacts}{
\entityCat, \relationCat}
\State$\groundTruthTriples \gets []$ \Comment{\commentstyle{Initialization}}
\For{$\m{entity}$ $\in$ \entityCat~} \Comment{\commentstyle{Iterate over each entity}}
\For{~$\nm$ $\in$  \relationCat~} \Comment{\commentstyle{Iterate over each relation}}
\State $\widetilde{\relation} \gets$ \Call{QueryDB}{$\m{entity}$, $\nm$} 
\Comment{\commentstyle{Retrieve ground facts}}
\State $\groundTruthTriples.\m{append}(\widetilde{\relation})$ \Comment{\commentstyle{Extend the ground facts}}
\EndFor
\EndFor
\State \Return $\groundTruthTriples$ \Comment{\commentstyle{Return the ground facts}}
\EndFunction
\end{algorithmic}
\end{algorithm}
}

\begin{figure}[!b]
\centering
\small
\begin{gather*}
\frac{
\begin{matrix}
\RNeg\\
{\drule}{=}\,\nm_{\m{new}}\,(\SUBJ, \OBJ){\hornarrow} !\nm\,(\SUBJ, \OBJ)
\end{matrix}
}{
\deriveRules{\nm}{\nm_{\m{new}}}{{\drule}}}
\ \  
\frac{
\begin{matrix}
\RInv\\
{\drule}{=}\,\nm_{\m{new}}\,(\SUBJ, \OBJ) {\hornarrow} \nm\,(\OBJ, \SUBJ)
\end{matrix}
}{
\deriveRules{\nm}{\nm_{\m{new}}}{{\drule}}}
\\[0.4em]
\frac{
\begin{matrix}
\RSym\\
{\drule}{=}\,\nm\,(\SUBJ, \OBJ) {\hornarrow} \nm\,(\OBJ, \SUBJ)
\end{matrix}
}{
\deriveRules{\nm}{\nm}{{\drule}}}
\quad\   
\frac{
\begin{matrix}
{\drule}{=}\,\nm\,(\SUBJ, \OBJ') {\hornarrow} 
\\ 
\nm\,(\SUBJ, \OBJ), 
\nm\,(\OBJ, \OBJ')
\end{matrix}
}{
\deriveRules{\nm}{\nm}{{\drule}}}  \RTrans
\end{gather*}
\caption{Deriving New Facts From the Known Facts}
\label{fig:basic_op_for_predicates}
\end{figure}

\vspace{-2mm}
\subsection{Deriving Simple Facts via Logical Reasoning}
\label{sec:derive_more_facts}
Based on the relation category, each predicate is labeled with a different set of properties, shown in \figref{table:relations}, which are mapped to different derivation rules. 

Based on the ground facts extracted from the databases, \tool derives additional facts to enrich the knowledge and generates test cases from each derived fact. 
As shown in \figref{fig:basic_op_for_predicates}, it provides four basic derivation rules, 
providing sound strategies to generate mutated facts from the ground facts. 
Note that these rules are also prevalently adopted in several literature~\cite{zhou2019completing, ren2020beta, liang2022reasoning, TIAN2022100159, abboud2020boxe} in the context of knowledge reasoning.
Given a predicate name $\nm$, the derivation  ``$\deriveRules {\nm}{\nm_{\m{new}}}{{\drule}}$'' holds if $\nm$ can be applied into a Prolog rule ${\drule}$, and produces more facts upon a new predicate with the name  $\nm_{\m{new}}$. 
These new predicates are freshly generated based on predefined suffixes.

As indicated in \figref{table:relations}, 
all the predicates can be applied to the $\RNeg$ rule, which derives the negated relations, e.g., ``!$\m{was}$'' using its positive counterpart, e.g., ``$\m{wasn't}$''. 
For the \emph{noun phrase} relations, both $\RSym$ and $\RTrans$ rules can apply, which generate more facts without creating new predicates. 
For the \emph{verb phrase} relations, both passive voice and active voice predicates can be applied to the  $\RInv$ rule, which captures the inverse relations, where the subject and object can be reversely linked through a variant of the original relation.  An example of such a rule is: 
\[
\m{influence(\OBJ, \SUBJ)}\hornarrow\m{influence\_by(\SUBJ, \OBJ)}
.\]

We summarize the fact derivation process using \algoref{alg:logic_reasoning}. Given any relation category, we iterate its predicates and generates the derivation rule $\drule$ (Line 4). 
For simplicity, we assume that the choice of which derivation rule to apply is predetermined. Based on this assumption, a new Prolog program is constructed, comprising ground facts and $\drule$. 
In particular, we use $\llbracket \relation \rrbracket_{\Prolog}$ to denote the query results of $\relation$ concerning the Prolog program $\Prolog$, and $\Prolog$ contains all the ground facts and the derivation rule. 
{Note that when $\relation$ contains no variables, it returns Boolean results indicating the presence of the fact; otherwise, it outputs all the possible instantiation of the variables. }
For each instantiation that contains one subject ``\Subj'' and one object ``\Obj'', we then compose them with the new predicate, which is taken as a  \emph{derived fact}.  
These derived facts are later used to generate NL test cases, detailed in \secref{prompt}.

{
\begin{algorithm}[!h]
\caption{Deriving New Facts}
\label{alg:logic_reasoning}
\small
\begin{algorithmic}[1]
\Require Ground Facts ($\groundTruthTriples$), Relation Category (\relationCat)
\Ensure Derived Facts ($\derivedFacts$)
\Function{DerivingFacts}{$\groundTruthTriples$, \relationCat}
\State $\derivedFacts \gets []$ \Comment{\commentstyle{Initialization}}
\For{$\nm$ in \relationCat}
\Comment{\commentstyle{Iterate each predicate}}
\State $\deriveRules{\nm}{\nm_{\m{new}}}{{\drule}}$\Comment{\commentstyle{Obtain the new predicate}} 
\State $\Prolog \gets \groundTruthTriples\plus\plus{\drule}$ \Comment{\commentstyle{Construct the prolog program}} 
\State $\m{instantiations} \gets \llbracket \nm_{\m{new}}(\SUBJ, \OBJ)\rrbracket_{\Prolog}$ 
\For{(\Subj, \Obj) in $\m{instantiations}$}
\Comment{\commentstyle{Iterate each entity tuple}}
\State $\relation_{\m{new}} \gets \nm_{\m{new}}(\Subj, \Obj)$ 
\Comment{\commentstyle{Construct the derived fact}}
\State $\derivedFacts.\m{append}(\relation_{\m{new}})$ \Comment{\commentstyle{Append the derived facts}}
\EndFor
\EndFor
\State \Return $\derivedFacts$ \Comment{\commentstyle{Return the derived facts}}
\EndFunction 
\end{algorithmic}
\end{algorithm}
}

\vspace{-2mm}
\subsection{Deriving Facts via Temporal Reasoning}
\label{sec:encoding_MTL}


Apart from the basic derivation rules, \tool can also automatically derive complex composition rules based on \emph{Metric Temporal Logic} (MTL) \cite{DBLP:conf/lics/OuaknineW05}. 
Specifically, we generate temporal test cases  based on randomly generated MTL formulae over historical events. 
We define the syntax for MTL formula in \figref{fig:syntax_of_the_metric_temporal_logic}, which contains the temporal operators for ``finally ($\mathcal{F}$)'', 
``globally ($\mathcal{G}$)'', 
``until ($\mathcal{U}$)'', 
and ``next ($\mathcal{N}$)''. 
The atomic propositions here are basic event relations $\nm$. 
The time intervals are pairs of natural numbers indicating the lower and upper bounds of the years
; and the constraint $\Istart \,{\leq}\, \Iend$ is enforced implicitly for all time intervals. 
In this paper for simplicity, we use discrete time measured in years as the smallest time interval. However, the framework can be extended to accommodate any smaller discrete time intervals, such as days or seconds.

\begin{figure}[!h] 
\vspace{-2.5mm}
\small
\centering
\begin{align*}
(\m{MTL})\quad & \mtl &{::=}\quad &
\nm 
\,{\mid}\, \mathcal{F}_\interval \,\mtl
\,{\mid}\, \mathcal{G}_\interval \,\mtl 
\,{\mid}\, \mtl_1  
\,\mathcal{U}_\interval \,  \mtl_2 
\,{\mid}\, \mathcal{N} \,\mtl
\,{\mid}\, 
\\
&&&
 \mtl_1  \,{\wedge}\, \mtl_2
\,{\mid}\, \mtl_1  \,{\vee}\, \mtl_2
\,{\mid}\, \neg \mtl 
\\[0.3em]
(\m{Time~Interval}) \quad& \interval &{::=}\quad & [\Istart, \Iend]
\end{align*}
\vspace{-4mm}
\caption{A Core Syntax of MTL}
\label{fig:syntax_of_the_metric_temporal_logic}
\vspace{-3mm}
\end{figure}

To facilitate the generation of temporal-based Q\&A pairs, we define the semantics model for the MTL formulae in \defref{def:semantics_MTL}, where the history is a set of facts. 
Here, we use $\history$ as a set of historical relations, 
e.g., ``$\nm\_{\m{TS}}(\interval, \Subj, \Obj)$'', which are the time-stamped version relations of the three-element relations ``$\nm(\Subj, \Obj)$'', derived by one of the following rules: \\[-0.5em]

\noindent 
{\small $\ \   
\nm\_{\m{TS}}(\interval, \Subj, \Obj) \hornarrow \nm(\Subj, \Obj), \m{start}(\Obj, n_1), \m{end}(\Obj, n_2), \interval{=}[n_1, n_2]. 
$}\\[-0.5em]

\noindent  {\small $\ \  \nm\_{\m{TS}}(\interval, \Subj, \Obj) \hornarrow \nm(\Subj, \Obj), \m{start}(\Subj, n_1), \m{end}(\Subj, n_2), \interval{=}[n_1, n_2].$}
\\

\noindent which construct the event intervals using the time stamps associated with the object or the subject, respectively. 
The ``$\m{start}$'' and ``$\m{end}$'' predicates are originally generated from the knowledge database and mark the starting and ending points of the duration of the object (or subject) event. 
For simplicity, we use ``$\nm\_{\m{TS}}(\interval)$'' to abbreviate ``$\nm\_{\m{TS}}(\interval, \Subj, \Obj)$'' when $\Subj$ and $\Obj$ are unambiguously unique from the context. 
We also use $\llbracket \nm\_{\m{TS}}(\interval) 
\rrbracket_{\history}$ to denote the validity of querying the presence of a fact $\nm\_{\m{TS}}(\interval)$ 
against the historical facts $\history$, which stores all the time-stamped three-element predicates.

\begin{definition}[A Point-based Semantics for MTL]
\label{def:semantics_MTL}
Given a set of (historical) facts $\history$, recording all the events that happened in history, an MTL formula $\mtl$, and a concrete time point  $\timepoint$, the satisfaction relation $(\history, \timepoint) \models \mtl$  (read at the time point \timepoint, the history $\history$ satisfies $\mtl$) is recursively defined as follows: 

{
\small
\begin{align*}
(\history, \timepoint) &\models 
\nm &\m{iff}&~ 
\m{\exists\,\interval}.~ 
\llbracket \nm\_{\m{TS}}(\interval) \text{$\rrbracket_{\history}$}{=}\m{true}
~\m{and}~
\timepoint\,{\in}\,\interval
\\[0.1em]
(\history, \timepoint) &\models \mathcal{F}_\interval \,\mtl & 
\m{iff}&~ 
\m{\exists\,\distance}.~\distance\,{\in}\,I  ~ \m{and}
~ (\history, \timepoint\plus\distance)\models\mtl
\\[0.1em]
(\history, \timepoint) &\models \mathcal{G}_\interval\,\mtl & 
\m{iff}&~ 
\m{\forall\,\distance}.~\distance\,{\in}\,I  ~ \m{and}
~ (\history, \timepoint\plus\distance)\models\mtl
\\[0.1em]
(\history, \timepoint) &\models \mathcal{N}\,\mtl & 
\m{iff}&~ 
(\history, \timepoint\plus 1)\models\mtl
\\[0.1em]
(\history, \timepoint) &\models\neg \mtl & \m{iff}&~
(\history, \timepoint)\not\models\mtl
\\[0.1em]
(\history, \timepoint) &\models \mtl_1 \, \mathcal{U}_\interval \,\mtl_2  & \m{iff}&~  \m{\exists\,\distance}.~ \distance\,{\in}\,\interval  ~ \m{and}~ (\history, \timepoint\plus\distance)\models\mtl_2 ~ \m{and}
\\[0.1em] 
&&& ~ 
\m{\forall}\, 
k~\m{with} ~\timepoint{<}k{<}(\timepoint\plus\distance), 
(\history, k)\models \mtl_1
\\[0.1em]
(\history, \timepoint) &\models\mtl_1 \, {\wedge} \,\mtl_2 & \m{iff}&~ (\history, \timepoint)\models\mtl_1 ~\m{and}~ (\history, \timepoint)\models\mtl_2
\\[0.1em]
(\history, \timepoint) &\models\mtl_1 \, {\vee} \,\mtl_2 & \m{iff}&~ (\history, \timepoint)\models\mtl_1 ~\m{or}~ (\history, \timepoint)\models\mtl_2 
\end{align*}}
\end{definition}
\vspace{2mm}

We randomly generate temporal test cases based on the rich set of historical events and the syntax templates defined in \figref{fig:syntax_of_the_metric_temporal_logic}. 
Each temporal question consists of a concrete MTL formula and a concrete time point, i.e., $(\phi, \timepoint)$. 
For example, the query ``\emph{At 1800, does Victorian era finally come within 40 years?}'' is represented as $(\mathcal{F}_{[0, 40]} \m{victorian\_era}, 1800)$. 
Next, we show how to obtain the expected answer by automatically generating Prolog reasoning rules. 

Given a query $\mtl$, the relation ``$\encoding{\mtl}{\nm}{\widetilde{\drule}}$'' holds if $\mtl$ can be translated into a set of Prolog rules, i.e., $\widetilde{\drule}$. 
Querying ``$\nm(\interval)$'' 
$\widetilde{\drule}$, against the known database facts yields a set of instantiation of the interval $\interval$. 
The validity of $\mtl$ at any given time point $\timepoint$ is then indicated by the existence of a concrete interval  $\interval$ such that $\timepoint\,{\in}\,\interval$. 
We define the full set of encoding rules for MTL operators in \figref{fig:encoding_rules_mtl}. 

These encoding rules deploy several auxiliary predicates: the 
``$\m{findall}(\interval, \nm)$'' relation indicates that $\interval$ is a union of all the time intervals which satisfy $\nm$; 
the  ``$\m{compl}(\interval, \interval_1)$'' relation indicates that time intervals $\interval$ and $\interval_1$ complement each other, and their union encompasses the entire set of time points; the union and intersection operations, denoted by $\cup$ and $\cap$, are applied to two sets of time intervals. 

\begin{figure}[!h]
\vspace{-2mm}
\vspace{0mm}
\begin{lstlisting}[xleftmargin=6em,numbersep=5pt,basicstyle=\footnotesize\ttfamily]
//nm1 = charles_dickens
charles_dickens_TS([1812, 1870]).
//nm2 = victorian_era
victorian_era_TS([1837, 1901]).
\end{lstlisting} 
\vspace{-1mm}
\caption{Database $\history_s$ Containing Two Time-stamped Events}
\label{fig:Prolog_encoding_Example}
\vspace{-2mm}
\end{figure}

Next, we illustrate the encoding rules for each MTL operator using a few examples. 
To facilitate the illustration, we use a small database $\history_s$ defined in \figref{fig:Prolog_encoding_Example}, which contains two facts: ``\emph{The author Charles Dickens was born in 1812 and he lived until 1870, which spanned a significant portion of the Victorian era}'' and ``\emph{The Victorian era started from 1837 until Queen Victoria died in 1901}'': 

\begin{figure}[!b]
\vspace{-3mm}
\centering\small
\begin{gather*} 
\input{encoding_rules/AP-R}
\\[0.6em]
\input{encoding_rules/Finally_encoding}
\\[0.6em]
\input{encoding_rules/Globally_encoding}
\\[0.6em]
\input{encoding_rules/Next_encoding}
\\[0.6em]
\input{encoding_rules/Until_encoding}
\\[0.6em]
\input{encoding_rules/Neg_encoding}
\\[0.6em]
\input{encoding_rules/Conj_encoding}
\\[0.6em]
\input{encoding_rules/Disj_encoding}
\end{gather*}
\caption{Encoding MTL Formula $\mtl$ using Prolog Rules}
\label{fig:encoding_rules_mtl}
\end{figure}

\begin{enumerate}[itemsep=0.7em,leftmargin=!,wide]
\item 
When 
$\mtl\,{=}\,\m{charles\_dickens}$ and $\timepoint\,{=}\,1800$: \\ 
According to the encoding rule $[\trans\text{-}\m{AP}]$, the generated Prolog rule is: $\m{\nm1(\interval)}\hornarrow\,\m{
charles\_dickens\_TS(\interval)}$.
Now, querying ``$\nm1(\interval)$'' against $\history_s$ yields $\interval\,{=}\,[1812, 1870]$. Since $1800\,{\not\in}\,\interval$, the expected result of this query is false.

Similarly, when 
$\mtl\,{=}\,\m{victorian\_era}$ and $\timepoint\,{=}\,1900$: \\ 
According to the encoding rule $[\trans\text{-}\m{AP}]$, the generated Prolog rule is: $\m{\nm2(\interval)}\hornarrow\,\m{
victorian\_era\_TS(\interval)}$.
Now, querying ``$\nm2(\interval)$'' against $\history_s$ yields $\interval\,{=}\,[1837, 1901]$. Since $1900\,{\in}\,\interval$, the expected result of this query is true. 

\item When $\mtl\,{=}\,\mathcal{F}_{[0, 40]}\,\m{victorian\_era}$ and $\timepoint\,{=}\,1800$: \\
According to the encoding rule $[\trans\text{-}\m{Finally}]$, the generated Prolog rule is: 
$\m{finally\_\nm2([n_1\text{-}40, n_2\text{-}0])}\hornarrow \m{
\nm2([n_1, n_2])}.$
Now, querying ``$\m{finally\_\nm2}(\interval)$'' against $\history_s$ yields $\interval\,{=}\,[1797, 1901]$. Since $1800\,{\in}\,\interval$, the expected result is true. 
Indeed, all the time points in $\interval$ satisfy the property: ``\emph{within 40 years, finally Victorian era came/still exist}''. 

\item When $\mtl\,{=}\,\mathcal{G}_{[30, 50]}\,\m{victorian\_era}$ and $\timepoint\,{=}\,1800$: \\
According to the rule $[\trans\text{-}\m{Globally}]$, the generated Prolog rule is: $\m{globally\_\nm2([n_1\text{-}30, n_2\text{-}50])}\hornarrow \m{
\nm2([n_1, n_2])}.$
Now, querying ``$globally\_\nm2(\interval)$'' against $\history_s$ yields $\interval\,{=}\,[1807, 1851]$. Since $1800\,{\not\in}\,\interval$, the expected result is false. 
Indeed, only all the time points in $\interval$ satisfy that ``\emph{Victorian era is globally true throughout the 30th to the 50th years in the future}''. 

\item When $\mtl\,{=}\,\mathcal{N}\,\m{victorian\_era}$ and $\timepoint\,{=}\,1836$: \\
According to the rule $[\trans\text{-}\m{Next}]$, the generated Prolog rule is: $\m{next\_\nm2([n_1\text{-}1, n_2\text{-}1])}\hornarrow \m{
\nm2([n_1, n_2])}.$
Now, querying ``$\m{next\_\nm2}(\interval)$'' against $\history_s$ yields $\interval\,{=}\,[1836, 1900]$. Since  $1836\,{\in}\,\interval$, the expected result is true. 
Indeed, all the time points in $\interval$ satisfy that ``\emph{next year Victorian era came/still exist}''. 

\item When $\mtl\,{=}\,\m{charles\_dickens}
~\mathcal{U}_{[10, 20]}\,\m{victorian\_era}$ and $\timepoint{=}1800$: This query aims to determine the time interval $\interval$ that encompasses all time points $\timepoint'$ for which there exists a future year ($\timepoint'\plus\distance$) when the Victorian era had begun; and during the time from $\timepoint'$ to $\timepoint'\plus\distance$, Charles Dickens must have been born and remained alive throughout. Lastly, check if $1800\,{\in}\,\interval$.

In $[\trans\text{-}\m{Until}]$, 
$\m{helper1}$ computes all the possible  $\timepoint'\plus d$  
such that $(\history, k)\models \m{charles\_dickens} ~\m{forall}~ 
k~\m{with} ~\timepoint'{<}k{<}(\timepoint'\plus\distance)$. 
Next $\m{helper2}$ computes the overlapping  intervals of $\m{helper1}$ and the intervals that also satisfy $(\history, \timepoint'\plus\distance)\,{\models}\,\m{victorian\_era}$. 
Then $\nm_f$ computes the interval of $\timepoint'$ which finally reach $\m{helper2}$ within 10 to 20 years. 
Lastly, the final answer of the interval of $\timepoint'$~is the intersection of $\nm_f$ and $\nm_1$. 
Therefore, given the concrete query $(\phi, \timepoint)$, from $[\trans\text{-}\m{Until}]$, 
the generated rules are shown in \figref{fig:until10-20-encoding}.

\begin{figure}[!h]
\vspace{-3mm}
{
\begin{align*}
&\m{helper1([n_1\plus10, n_2\plus1])}\hornarrow \m{\nm1([n_1, n_2])}.
& // [1822, 1871]
\\
&\m{helper2(\interval_1\cap\interval_2)}\hornarrow\m{helper1(\interval_1)}, \m{\nm2(\interval_2)}. 
& // [1837, 1871]
\\
& \m{\nm_f([n_1\text{-}20, n_2\text{-}10])}\hornarrow \m{
helper2([n_1, n_2])}.
& // [1817, 1861]
\end{align*}
\vspace{-4mm}
\[\m{charles}\_\m{until}\_\m{victorian\_era}(\interval_1\cap\interval_2) \hornarrow \m{\nm1(\interval_1)}, \m{\nm_f(\interval_2)}.\]}
\caption{Prolog Rules Generated for an "Until" Query}
\label{fig:until10-20-encoding}
\vspace{-1mm}
\end{figure}

Querying ``$\m{charles}\_\m{until}\_\m{victorian\_era}$'' against $\history_s$ yields $\interval\,{=}\,[1817, 1861]$. Since $1800\,{\not\in}\,\interval$, the expected result is false. 
Indeed, only all the time points in $\interval$ satisfy $\phi$ under the semantic definition of \emph{Until}, cf.  \defref{def:semantics_MTL}. For example when $\timepoint'{=}1817$, there exists $\distance{=}20$ such that $\phi$ holds; and when $\timepoint'{=}1861$ there exists $\distance{=}10$ such that $\phi$ holds. 

Note that, in this encoding, the interval of ``\emph{Until}'' operators does not include $[0, 0]$, as $\mtl_1  
\,\mathcal{U}_{[0, 0]} \,  \mtl_2$ essentially equals $\mtl_2$. Therefore, when the interval compasses $[0, 0]$, we use the following rule to decompose the encoding: (Note that when $\interval'\,{=}\,\interval{\setminus}[0, 0]$, it means $\interval'\cup[0, 0]\,{=}\,\interval$)
\begin{align*}
\frac{
[0, 0] \subseteq \interval 
\qquad 
\interval^\prime \,{=}\, \interval{\setminus}[0, 0]
}{
\mtl_1  
\,\mathcal{U}_\interval \,  \mtl_2 
\equiv  (\mtl_1  
\,\mathcal{U}_{\interval^\prime} \,  \mtl_2)  \vee \mtl_2 
} \ [\trans\text{-}\m{Until}\text{-}0]
\end{align*}

\vspace{2mm}
\item When $\mtl\,{=}\,\neg\,\m{victorian\_era}$ and $\timepoint\,{=}\,1800$: \\
By $[\trans\text{-}\m{Neg}]$, the generated Prolog rule is: 
$\m{neg\_\nm2(\interval)}\hornarrow$ 
$\m{findall}(\interval_1, \nm1), \m{compl}(\interval_1, \interval).$ 
Querying ``$\m{neg\_\nm2}(\interval)$'' against $\history_s$ yields $\interval\,{=}\,[1, 1836] \cup [1902, 2024].$
Here, we take all the after-century years to be the full set. 
Since $1800\,{\in}\,\interval$, the expected result is true. 
Indeed, all the time points in $\interval$ satisfy that ``\emph{Victorian era has not come/already passed}''.

\item When $\mtl\,{=}\,\m{charles\_dickens}\,{\wedge}\,\m{victorian\_era}$ and $\timepoint{=}1900$: 
By $[\trans\text{-}\m{Conj}]$, the generated Prolog rule is: 
$\m{\nm1\_and\_\nm2(\interval_1{\cap}\interval_2)}\hornarrow 
\m{findall}(\interval_1, \nm1), \m{findall}(\interval_2, \nm2)$. 
Now, querying ``$\m{\nm1\_and\_\nm2}(\interval)$'' against $\history_s$ yields $\interval\,{=}\,[1837, 1870]$. Since $1900\,{\not\in}\,\interval$, the expected result is false. 
Indeed, only the time points in $\interval$ satisfy that ``\emph{Victorian era exists and Charles Dickens is alive}''.

\item When $\mtl\,{=}\,\m{charles\_dickens}\,{\vee}\,\m{victorian\_era}$~and
$\timepoint{=}1900$: 
By $[\trans\text{-}\m{Disj}]$, the generated Prolog rule is as follows: 
$\m{\nm1\_or\_\nm2(\interval_1{\cup}\interval_2)}\hornarrow 
\m{findall}(\interval_1, \nm1), \m{findall}(\interval_2, \nm2)$. 
Now, querying ``$\m{\nm1\_or\_\nm2}(\interval)$'' against $\history_s$ yields $\interval\,{=}\,[1812, 1901]$. Since $1900\,{\in}\,\interval$, the expected result is true. 
Indeed, all the time points in $\interval$ satisfy that ``\emph{Victorian era exists or Charles Dickens is alive}''.  
\end{enumerate}
\vspace{3mm}

{\emph{\textbf{Remark.}}} 
While discrete-time MTL is commonly employed for model-checking timed verification~ \cite{DBLP:phd/us/Henzinger91}, utilizing Prolog to encode MTL for reasoning about the temporal relationships among events and detecting LLM hallucination is novel. 
These encoding rules can recursively accommodate the entire range of MTL formulas, including those with any level of nesting. 
We provide a formal definition for the correctness of these encoding rules in \theoref{ThemSoundAndComplete} and demonstrate that they are both sound and complete.

\begin{restatable}[Correctness of the encoding rules]{thm}{ThemSoundAndComplete}
\label{ThemSoundAndComplete}
~\\
Given any $\history$, 
$\mtl$, and 
$\encoding {\mtl}{\nm}{\widetilde{\drule}}$, let $\Prolog{=}\history \plus\plus \widetilde{\drule}$, we define,  \\
(1) Soundness: \\
$\forall\, \interval$.  
$\llbracket \nm(\interval) \rrbracket_{ \Prolog} {=} \m{true}$, then 
$\forall\, \timepoint\,{\in}\, \interval$, we have 
$(\history, \timepoint) \models \mtl$;  \\
(2) Completeness: \\ 
$\forall \,\timepoint\,$. $(\history, \timepoint) \models \mtl$, then $\exists\, \interval$. $\llbracket \nm(\interval) \rrbracket_{\Prolog} {=} \m{true}$  and $\timepoint\,{\in}\,\interval$. 
\end{restatable}

\begin{proof}
By structural induction over $\phi$. 
The detailed proofs are given in the Appendix. 
\end{proof}

\begin{table*}[!t]
\setlength{\tabcolsep}{1pt}
\centering
\caption{Relation-Template Mapping Patterns.}
\label{table:template}
\footnotesize
\begin{tabular}{l l}
\toprule 
\textbf{Relation} & \textbf{Template Examples}  \\
    \midrule
{Noun Phrase} & \begin{tabular}[l]{@{}l@{}} - Is it true that 
$\langle \m{Subject}\rangle$ and 
$\langle\m{Object}\rangle$ share 
$\langle\m{Relation}\rangle$? 
\\ - $\langle\m{Subject}\rangle$ and $\langle\m{Object}\rangle$ have/made/shared totally different $\langle\m{Relation}\rangle$. Please judge the truthfulness of this statement.
    \end{tabular}  \\
    \midrule
    \begin{tabular}[l]{@{}l@{}} Verb Phrase \\ (Passive Voice) \end{tabular} & \begin{tabular}[l]{@{}l@{}} - Is it true that $\langle \m{Subject}\rangle$ is/was/are/were $\langle \m{Relation}\rangle$ $\langle\m{Object}\rangle$? \\ - It is impossible for $\langle \m{Subject}\rangle$ to be $\langle\m{Relation}\rangle$ $\langle\m{Object}\rangle$. Am I right?
    \end{tabular}  \\
    \midrule
\begin{tabular}[l]{@{}l@{}} Verb Phrase \\ (Active Voice) \end{tabular}
 & \begin{tabular}[l]{@{}l@{}} - Is it true that 
 $\langle \m{Subject}\rangle$
 $\langle\m{Relation}\rangle$
 $\langle\m{Object}\rangle$?  \\ - $\langle \m{Subject}\rangle$ $\langle\m{Relation}\rangle$ $\langle\m{Object}\rangle$. 
 \end{tabular}  \\

    \bottomrule 
\end{tabular}
\end{table*}

\begin{table*}[!t]
\setlength{\tabcolsep}{3pt}
\centering
\caption{Temporal-Template Mapping Patterns (implicitly querying upon year $\iyear$).}
\label{table:temporal_template}
\footnotesize
\begin{tabular}{l  l }
\toprule 
\textbf{MTL Formulae} & \textbf{Template Examples}  \\
\midrule 
\mtltoNL($\nm$) &  Did $\langle\nm\rangle$ happen at year $\langle \iyear \rangle$? 
\\  \midrule
\mtltoNL($\mathcal{F}_\interval \,\mtl$) & Did ``Event'' finally happen within the time frame of $\langle \interval \rangle$ after the year $\langle \iyear \rangle$, where ``Event'' is defined as $\langle \mtltoNL(\mtl)\rangle$? 
\\ 
\midrule 
\mtltoNL($\mathcal{G}_\interval \,\mtl$) &  Did ``Event'' globally happen within the time frame of $\langle \interval \rangle$ after the year $\langle \iyear \rangle$, where ``Event'' is defined as $\langle \mtltoNL(\mtl)\rangle$?
\\ 
\midrule 
\mtltoNL($\mathcal{N} \,\mtl$) & Did ``Event'' happen in the next year of $\langle \iyear \rangle$, where ``Event'' is defined as $\langle \mtltoNL(\mtl)\rangle$? 
\\ 
\midrule 
\mtltoNL($\mtl_1 \, \mathcal{U}_\interval \,\mtl_2$) &  
\begin{tabular}[l]{@{}l@{}} 
Did ``Event$_1$'' happen continuously until ``Event$_2$'' started, during the period $\langle \interval \rangle$ after the year $\langle \iyear \rangle$, \\ 
where ``Event$_1$'' is $\langle \mtltoNL(\mtl_1) \rangle$ and ``Event$_2$'' is $\langle \mtltoNL(\mtl_2) \rangle$?
\end{tabular}
\\ 
\midrule 
\mtltoNL($\mtl_1  \,{\wedge}\,  \mtl_2$) &  Did both ``Event$_1$'' and  ``Event$_2$'' happen at year $\langle \iyear \rangle$, where ``Event$_1$'' is $\langle \mtltoNL(\mtl_1) \rangle$ and ``Event$_2$'' is $\langle \mtltoNL(\mtl_2) \rangle$? 
\\ 
\midrule 
\mtltoNL($\mtl_1  \,{\vee}\,  \mtl_2$) &  Did either ``Event$_1$'' or ``Event$_2$''  happen at year $\langle \iyear \rangle$, where ``Event$_1$'' is $\langle \mtltoNL(\mtl_1) \rangle$ and ``Event$_2$'' is $\langle \mtltoNL(\mtl_2) \rangle$? 
\\ 
\midrule 
\mtltoNL($\neg \mtl $) &  Did ``Event'' not happen at year $\langle \iyear \rangle$, where ``Event'' is $\langle \mtltoNL(\mtl) \rangle$? 
\\ 
\bottomrule 
\end{tabular}
\end{table*}

\subsection{Benchmark Construction}
\label{prompt}
\tool constructs question-answer~(Q\&A) pairs and prompts to facilitate the testing for FCH. 
To address the challenge of the high human efforts required in the test oracle generation, we design an automated approach based on mapping relations between various entities to problem templates, largely reducing reliance on manual efforts. 

\textbf{\emph{Question Generation.}}
To facilitate efficient and systematic generation of test cases and prompts, we have adopted a method that leverages entity relationships and mappings of temporal operators to predefined Q\&A templates. 

When constructing relation-based Q\&A templates (without temporal operators), one key aspect lies in aligning various types of relations with the corresponding question templates from the mutated triples, i.e., the predicate type in the triple. Different relation types possess unique characteristics and expressive requirements, leading to various predefined templates. 
As listed in \tabref{table:template}, we map the relation types to question templates based on speech and the grammatical tense of the predicate to guarantee comprehensive coverage. 

When constructing temporal-logic-related queries, we define a mapping pattern for each temporal operator, as outlined in \tabref{table:temporal_template}. For any query expressed as ``$\mtl$''  with any concrete year $\iyear$ in query, the $\mtltoNL(\mtl)$ function converts the MTL formula $\mtl$ into a natural language query. In this context, $\mtltoNL(\mtl')$ is called recursively to generate the natural language description for the CTL subformula $\mtl'$.

In both mapping patterns, we enhance the construction of the naturally formatted questions by leveraging an LLM to reformulate the structure and grammar of the Q\&A pairs. 

\textbf{\emph{Answer Generation.}}
We note that the answer to the corresponding question is readily attainable from the factual knowledge and the Prolog reasoning rules, defined in both \figref{fig:basic_op_for_predicates} and \figref{fig:encoding_rules_mtl}. 
Primarily, it is easy to determine whether the answer is \emph{true} or \emph{false} based on the mutated triples and the ground-truth time intervals using temporal reasoning. 
Meanwhile, mutated templates with positive and negative semantics via the usage of synonyms or antonyms, which greatly enhance the diversity of the questions, can be treated similarly as the negation rule defined in \figref{fig:basic_op_for_predicates}. 
Specifically, if the answer to a question with original semantics is Yes/No, then for a question with mutated opposite semantics, the corresponding answer would 
be the opposite, i.e., No/Yes. For example, after obtaining the original Q\&A pair \textit{- ``Is it true that Crohn's disease and Huntington's disease could share similar symptoms and signs? - Yes.''}, we can use antonyms to mutate it into \textit{- ``Is it true that Crohn's disease and Huntington's disease have different symptoms and signs? - No.''}

\textbf{\emph{Prompt Construction.}}
As illustrated in \tabref{table:prompt}, before initiating any interaction with LLMs, we predefine specific instructions and prompts, requesting the model to utilize its inherent knowledge and inferential capabilities to deliver explicit (Yes/No/I don't know) judgments on our queries. Additionally, we instruct the model to present its reasoning process in a template following the judgment. The main goal is to ensure that LLMs give easy-to-understand responses using standardized prompts and instructions. 
This approach also helps the model to effectively exercise its reasoning abilities based on the given instructions and examples.

\begin{table*}[!t]
    \setlength{\tabcolsep}{1ex}
	\centering
	\small
	\caption{Prompt Template. 
    }
        \label{table:prompt}
        \vspace{-0.1cm}
	\begin{tabular}{l}
    \toprule 
    \rowcolor{mycolor}
    \textbf{\instruction:} Answer the question with your knowledge and reasoning power.\\
    \midrule
    \rowcolor{mycolor} \textbf{\query (Relation):}  Given the $\langle \textit{question} \rangle$: \textit{question}, please provide an answer with your knowledge and reasoning power.\\ 
    \rowcolor{mycolor} Think of it step by step with a human-like reasoning process. After giving the answer, list the knowledge used in your\\ 
    \rowcolor{mycolor} reasoning process in the form of declarative sentences and point by point. The answer must contain `Yes', `No' or `I \\
    \rowcolor{mycolor} don't know' at the beginning. \\
    \midrule
    \rowcolor{mycolor} \textbf{\query (Temporal):}  Given the question: 
    $\langle \textit{question} \rangle$, please provide an answer with your knowledge and reasoning power \\
    \rowcolor{mycolor}  upon metric temporal logic. Think it step by step with a human-like reasoning process. After giving the answer, list the \\
    \rowcolor{mycolor} evidence from your temporal reasoning  in the form of declarative sentences and point by point. The answer must contain   \\
\rowcolor{mycolor} `Yes', `No' or `I don't know' at the beginning.\\
    \bottomrule 
\end{tabular}
\end{table*}

\vspace{-1mm}
\subsection{Response Evaluation}\label{response}

The objective of this module is to enhance the detection of FCH in LLM outputs, specifically focusing on identifying the discrepancies between LLM responses and the verified ground truths. Recognizing the inherent limitation in directly accepting ``Yes'' or ``No'' answers from LLMs, our approach underscores the automated detection of factual consistency during the reasoning process presented by LLMs. 
Our approach is structured around several critical steps, as listed in \algoref{alg:eval} and detailed below:
\begin{enumerate}[wide,  labelindent=9pt]
\item \textbf{Preliminary Screening.} Given the LLM response $\llmResponse$, we first eliminate scenarios when the LLM declines to provide an answer, indicated by the ``answer'' field of LLM's responses. 
Most of these responses arise because the LLM lacks the relevant knowledge for the reasoning process. As these responses adhere to the LLM's principle of honesty, we categorize them as correct and normal responses.

\item \textbf{Response Parsing and Semantic Structure Construction.} Provided with the remaining suspicious responses from $\llmResponse$ and ground truth facts $\groundTruthTriples$, we use \textsc{ExtractTriple} function to extract triples that follow the same structure as the fact defined in the \secref{subsec3.1}. For each LLM response, the extracted triples ($\widetilde{\m{Trpl}}$) are based on the statements contained in the \textit{reasoning process} part of the LLM's response, which is further utilized to construct a response semantic structure $\semantic_{\m{resp}}$ using the \textsc{BuildGraph} function. In this structure, the $\widetilde{\entity}$ are depicted as \emph{nodes} ($\m{N}$), and the relational predicate ($\nm$) between them are illustrated as \emph{edges} ($\m{E}$). Concurrently using the same approach, we construct another semantic structure $\semantic_{\m{ground}}$ using $\groundTruthTriples$.

\item \textbf{Similarity-based Metamorphic Testing and Oracles.} 
We apply metamorphic relations to identify hallucination answers from LLMs, i.e., comparing the similarity between semantic structures generated by LLMs and the ground truth counterparts. Note that we provide four classifications: correct responses ($\m{CO}$), hallucinations caused by error inference ($\m{EI}$), hallucinations caused by erroneous knowledge ($\m{EK}$), and hallucinations containing both issues ($\m{OL}$). 
Specifically, the oracles for metamorphic testing can be divided into the following types:
 
\end{enumerate}

\begin{algorithm}[!b]
\caption{Response Evaluation}
\label{alg:eval}
\small
\begin{algorithmic}[1]
\Require LLM Response ($\llmResponse$), Ground Facts ($\groundTruthTriples$), Threshold ($\theta_{\m{e}}, \theta_{\m{n}}$)
\Ensure Evaluation Result Category~($CO, EK, EI, OL$)
\Function{EvaluateResponse}{$\llmResponse$, $\groundTruthTriples$, $\theta_{\m{e}}$, $\theta_{\m{n}}$}
    \State $CO, EK, EI, OL \gets []$ \Comment{\commentstyle{Initialization}}
    \If{$\llmResponse.answer = refusal$}
        \State
        $CO.\m{append}(\llmResponse)$ \Comment{\commentstyle{Preliminary Screening}}
    \Else
        \State $\widetilde{\m{Trpl}} \gets$ \Call{ExtractTriple}{$\m{Resp.reasoning}$} 
        \State $\semantic_{\m{resp}}, \semantic_{\m{ground}} \gets$ \Call{BuildGraph}{$\widetilde{\m{Trpl}}, \groundTruthTriples$} 
        \State $\m{s}_{\m{e}} \gets$ $\similarity_\m{e}${$(\semantic_{\m{resp}}$, $\semantic_{\m{ground}})$} \Comment{\commentstyle{Calculate edge similarity}}
        \State $\m{s}_{\m{n}} \gets$ $\similarity_\m{n}${$(\semantic_{\m{resp}}$, $\semantic_{\m{ground}})$} \Comment{\commentstyle{Calculate node similarity}}
        \If {$s_e < \theta_{e}$ and $s_n < \theta_{n}$}
            \State
            $OL.\m{append}(\llmResponse)$  \Comment{\commentstyle{Append  overlapped cases}}
        \ElsIf{$\m{s}_{\m{e}} < \theta_{\m{e}}$}
            \State 
            $EI.\m{append}(\llmResponse)$  \Comment{\commentstyle{Append error inference}}
        \ElsIf{$\m{s}_{\m{n}} < \theta_{\m{n}}$}
            \State 
            $EK.\m{append}(\llmResponse)$  \Comment{\commentstyle{Append error knowledge}}
        \Else
            \State
            $CO.\m{append}(\llmResponse)$
            \Comment{\commentstyle{Append correct response}}
        \EndIf
    \EndIf
    \State \Return $CO, EK, EI, OL$ 
\EndFunction
\end{algorithmic}
\end{algorithm}

\textbf{Edge Vector Metamorphic Oracle ($MO_E$)}: This oracle is based on the similarity of edge vectors between $\semantic_{\m{resp}}$ and $\semantic_{\m{ground}}$. If the vector similarity ($\m{s}_{\m{e}}$) between the edges in the $\semantic_{\m{resp}}$ and those in $\semantic_{\m{ground}}$ falls below a predetermined threshold $\theta_{\m{e}}$, it indicates that the LLM's answer significantly diverges from the ground truth. This suggests the presence of an FCH, and vice versa. 
More specifically, we utilize \emph{Jaccard Similarity}~\cite{J_S} to calculate the similarity score between edge vectors extracted from $\semantic_{\m{resp}}$ and  $\semantic_{\m{ground}}$. 
$$
\similarity_{\m{e}}(\semantic_{\m{resp}}, \semantic_{\m{ground}}) = \frac{|\widetilde{E}_{\m{resp}} \cap \widetilde{E}_{\m{ground}}|}{|\widetilde{E}_{\m{resp}} \cup \widetilde{E}_{\m{ground}}|}, $$check if $  \similarity_{\m{e}}(\semantic_{\m{resp}}, \semantic_{\m{ground}})  < \theta_{\m{e}} \ 
$~
where $\widetilde{E}_{\m{resp}}$ and $\widetilde{E}_{\m{ground}}$ denote the set of edges extracted from $\semantic_{\m{resp}}$ and $\semantic_{\m{ground}}$. 



\textbf{Node Vector Metamorphic Oracle ($MO_N$)}: This relation examines the similarity of node vectors between $\semantic_{\m{resp}}$ and $\semantic_{\m{ground}}$. 
Defined in a similar manner as $MO_E$, if the node similarity between the nodes ($\m{s}_{\m{n}}$) in the $\semantic_{\m{resp}}$ and those in $\semantic_{\m{ground}}$ falls below a predetermined threshold $\theta_{\m{n}}$, it indicates that the LLM's answer significantly diverges from the ground truth, and vice versa.
$MO_N$ can be captured by the Jaccard Similarity, defined as follows:

$$\similarity_{\m{n}}(\semantic_{\m{resp}}, \semantic_{\m{ground}}) = \frac{|N_{\m{resp}} \cap N_{\m{ground}}|}{|N_{\m{resp}} \cup N_{\m{ground}}|}, $$check if $
\similarity_{\m{n}}(\semantic_{\m{resp}}, \semantic_{\m{ground}})  < \theta_{\m{n}}  
$~
where $N_{\m{resp}}$ and $N_{\m{ground}}$ denotes the set of nodes extracted from $\semantic_{\m{resp}}$ and $\semantic_{\m{ground}}$.
Note that when determining the joint and union of the edges/nodes sets, we consider two edges/nodes as identical if their corresponding entities are identical or synonymous, and vice versa.

\section{Evaluation}\label{sec:eval}
Our evaluation targets the following research questions:

\begin{itemize}[wide]
\item \textbf{RQ1} (Effectiveness): How effective is \tool for in generating test cases and identifying LLM FCH issues?

\item \textbf{RQ2} (Hallucination Categorization and Analysis): What are the categorizations of LLM FCH issues? 

\item \textbf{RQ3} (Ablation Study): Whether the four types of logic reasoning rules 
can identify LLM FCH issues independently? 

\item \textbf{RQ4} (FCH Caused by Temporal Reasoning): 
How effective is \tool in detecting temporal-logic-related hallucinations; what are their categorizations; 
and which temporal operators trigger the most/least hallucinations? 
\end{itemize}



\subsection{Experimental Setup}\label{sec:ex_setup}
\noindent \textbf{Knowledge Extraction.} 
We use Wikipedia and Wikidata as sources to extract entities and structured information as base factual knowledge. After downloading the latest Wikipedia dump, we employ wikiextractor~\cite{Wikiextractor2015} to extract relevant text from Wiki pages. In parallel, we invoke Wikidata's SPARQL~\cite{sparql} query module to extract facts. 
Through data processing involving simplification and filtration, we amass a collection of basic factual knowledge encompassing 54,483 entities and 1,647,206 facts. 

\noindent \textbf{Logic Reasoning Processor.} 
For the logic reasoning module, we apply SWI-Prolog~\cite{wielemaker2012swi}, an open-source advanced logical programming interpreter. To effectively prevent errors due to excessive stacked strings, and ensure the proper operation of the logical processor when inserting a large number of facts into Prolog, we employ a sampling method and extract a subset of entities to form a query. 

\noindent \textbf{Models Under Test.} 
To guarantee a reliable evaluation for the RQs, we evaluate nine state-of-the-art large language models with \tool. Considering the diverse nature of LLMs, we select two distinct categories: the first category comprises API-accessible models with closed-source architecture including ChatGPT (gpt-3.5-turbo-0613), GPT-4, and GPT-4o~\cite{OpenAI2023GPT4TR}, and the second category consists of open-source LLMs with deployability, including Llama2-7B-chat-hf, Llama2-70B-chat-hf~\cite{touvron2023llama}, Mistral-7B-Instruct-v0.2~\cite{jiang2023mistral}, Mixtral-8x7B-Instruct-v0.1~\cite{jiang2024mixtral}, Llama3.1-8B-Instruct, and Llama3.1-70B-Instruct~\cite{llama32024tr}.  

\noindent \textbf{Response Validation Threshold $\theta$.} 
We apply StanfordOpenIE~\cite{angeli-etal-2015-leveraging} for knowledge triple extraction from LLM responses and then use SentenceBERT~\cite{reimers2019sentence} to calculate the vector similarity of nodes and edges from the constructed semantic structures. We also utilize GPT-4o to extract triples for some complex responses that StanfordOpenIE cannot handle effectively. Here, we set the threshold to 0.8, considering knowledge triples as semantically equivalent if they exceed this threshold and vice versa. To determine the threshold value, we sample 30 test cases and corresponding LLM responses from each of the nine knowledge domains listed in \figref{table:categories}. Through this analysis, we find that by setting the threshold values for both $\theta_\m{e}$ and $\theta_\m{n}$ at 0.8, with the given 270 test cases that are correctly classified, we can estimate the true positives among all test cases through \textit{Laplace's approach in the Sunrise problem}~\cite{laplace1951philosophical}, resulting in 99.6\% when detecting non-equivalent LLM answers as FCHs. In other words, all instances where an LLM's answer has a semantic similarity score below 0.8 compared to the ground truth were correctly identified as FCH cases. 

\noindent \textbf{Running Environment.} 
Our experiments are conducted on a server running Ubuntu 22.04 with two 64-core AMD EPYC 7713, 512 GB RAM, and two NVIDIA A100 PCIe 80GB GPUs. Our experiments consume a total of 240 GPU hours. 


\subsection{RQ1: Effectiveness}
To reveal the effectiveness of \tool, we evaluate the statistics of test cases generated by \tool and then evaluate the capabilities of identifying LLM FCH issues with the generated test cases. 
To further assess the effectiveness of test cases for uncovering FCH issues in specific knowledge domains, we evaluate the performances of LLMs on test cases across various knowledge domains.

\head{Effectiveness on Generating (Non-Temporal) Q\&A Test Cases.} We apply \tool to generate a Q\&A test benchmark, amounting to a comprehensive total of 7,200 test cases, designed to provide a broad and detailed evaluation of LLM FCH issues across specific knowledge domains. 

\begin{figure}[!t]
\hspace{-4mm}
\centering
\begin{subfigure}{}
\centering
\includegraphics[width=1\linewidth]{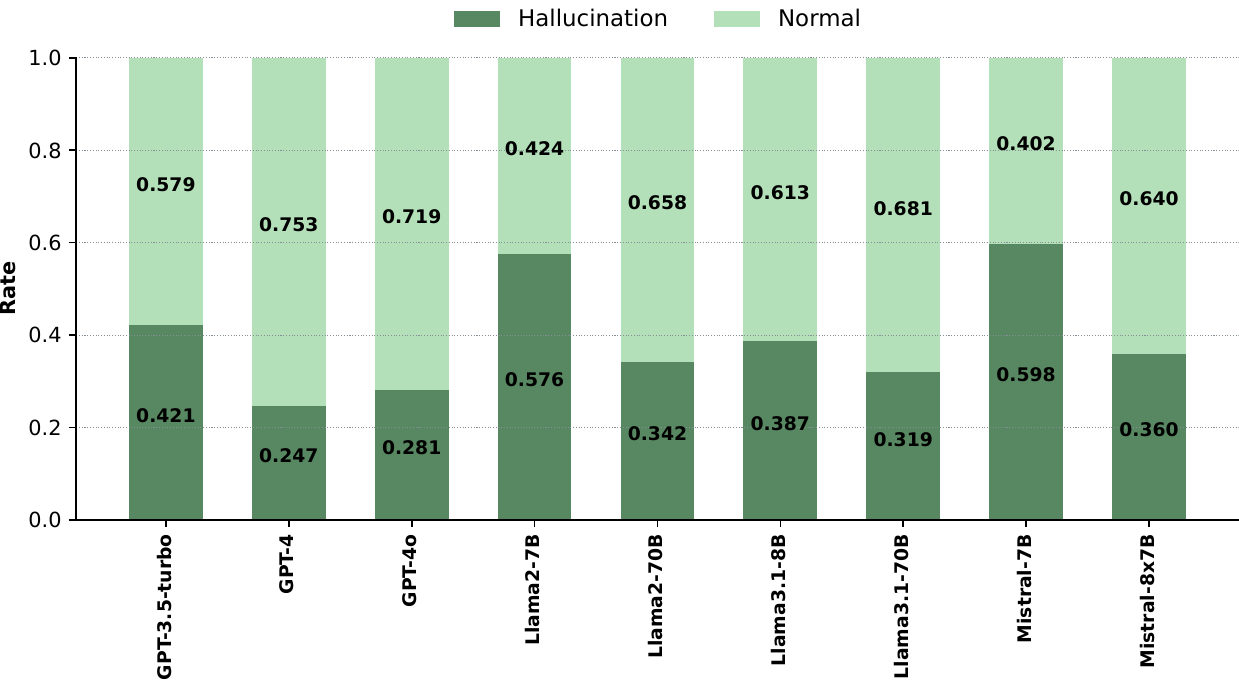}\\
\vspace{-0.2cm}
\caption{Overall Hallucination Rate of Various LLMs}
\vspace{0.3cm}
\label{fig:overall}
\end{subfigure}
\begin{subfigure}{}
\centering
\includegraphics[width=0.9\linewidth]{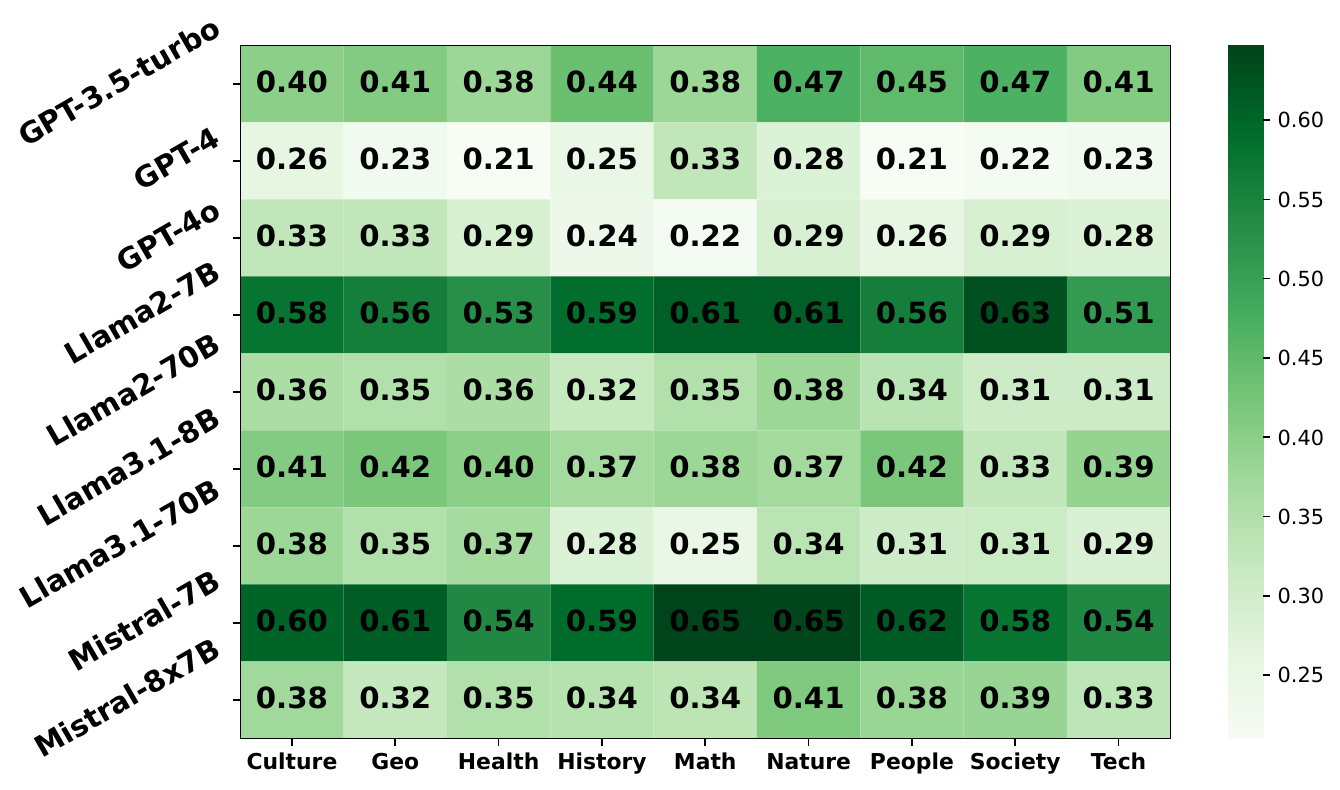}
\caption{Hallucination Rate Heatmap of Specific Domain}
\label{fig:rq1.2}
\end{subfigure}
\vspace{-3mm}
\end{figure}

\head{Effectiveness across LLMs.} We instruct LLMs under test utilizing Q\&A pairs derived from \tool, and automatically label both hallucination and normal responses. The results are presented in \figref{fig:overall}, illustrating the proportion of hallucination responses versus normal responses from LLMs under test.

Among all models, GPT-4 exhibits the best performance, demonstrating the lowest proportion of hallucinatory responses in the test cases generated by \tool, at only 24.7\%, while GPT-4o and ChatGPT fall slightly behind with 
28.1\% and 42.1\%. Open-source LLMs including Llama2-7B-chat-hf, Llama3.1-8B-Instruct, and Mistral-7B-Instruct-v0.2 with fewer parameters perform worse, but their counterparts with larger parameters (i.e., Llama2-70B-chat-hf, Llama3.1-70B-Instruct, and Mixtral-8x7B-Instruct-v0.1) achieve higher normal response rates surpassing ChatGPT on \tool. 
This indicates that the test cases generated by \tool successfully trigger hallucination responses across various LLMs when confronted with questions requiring logical reasoning capabilities.

\head{Effectiveness on Specific Domain Knowledge for Each LLM.}
To further explore the effectiveness of \tool in identifying FCH issues spanning various domains of LLMs, we compare hallucination response across nine specific domain knowledge. 
\figref{fig:rq1.2} presents the generated heatmaps of the confusion matrices for the specific knowledge field hallucination response rate based on the testing results.
It can be clearly observed that different models exhibit varying strengths and weaknesses across distinct knowledge domains. 
An interesting finding is that, within the domains of natural sciences and mathematics, LLMs generally exhibit weaker performance. This is potentially because there are many astrophysical or mathematical entities and their interrelationships in generated test cases by \tool. To answer such questions, the LLM needs an extensive understanding of astrophysical knowledge and mathematical theory. Thus, we infer that this realm of knowledge is not well-covered in the training datasets of LLMs under test, thereby resulting in high hallucination rates. Such a disparity in knowledge is likely a significant factor in the observed underperformance of LLMs in these specific domains.


\begin{tcolorbox}[title=ANSWER to RQ1, boxrule=0.8pt,boxsep=1.5pt,left=2pt,right=2pt,top=2pt,bottom=1pt]

Our evaluation using \tool reveals that existing LLMs have a notable tendency to produce FCH when faced with logical reasoning challenges, with hallucination rates ranging from 24.7\% to 59.8\%. The results varied across knowledge domains, highlighting that LLMs are more prone to FCH when answering questions that require highly specialized, domain-specific knowledge. 

\end{tcolorbox} 

\subsection{RQ2: FCH Categorization and Analysis}
\subsubsection{FCH Categorization}
We categorize the hallucination responses in more detail and focus primarily on the two types: error knowledge response, error inference response, and contradictory response. Note that we consider refusal to respond, such as `I don't know' due to the lack of relevant knowledge to be adhering to LLMs' honesty and truthfulness principles. Therefore, we categorize refusal to respond as a normal response. 
To ensure fair and unbiased categorization, 100 hallucination-related responses were randomly selected and independently categorized by three co-authors, who then discussed the results to reach a consensus categorization.

\head{Error Knowledge Response.} Originated from LLMs utilizing erroneous or contextually inappropriate knowledge during the reasoning process.

\head{Error Inference Response.} The most frequently occurring type is attributed to the lack of reasoning power and flawed reasoning thoughts of LLMs.




\subsubsection{Hallucination Measurement} 
Here, we provide the distribution of the hallucination categorization results, as demonstrated in \figref{fig:rq2}. There is partial overlap between these two types of hallucinations because incorrect reasoning processes may also involve erroneous knowledge. Among these issues, there are several contradictory answers primarily arising from inconsistency between incorrect reasoning processes and correct answers; thus, it exists in these two types of errors. Error inference hallucination presents the most, totalling half of the results on average. 
This indicates that the primary cause of FCH issues in logical reasoning is the insufficient reasoning capability of LLMs.
Besides, error knowledge adopted by LLMs during the logical reasoning process leads to 42.0\% 
FCH issues on average. The overlaps account for 7.9\%-21.1\% at the hallucination ratio, which indicates there are entities where LLMs have learned entirely erroneous relevant information, necessitating the employment of certain measures for correction.

\begin{figure}
\centering
\includegraphics[width=\linewidth]{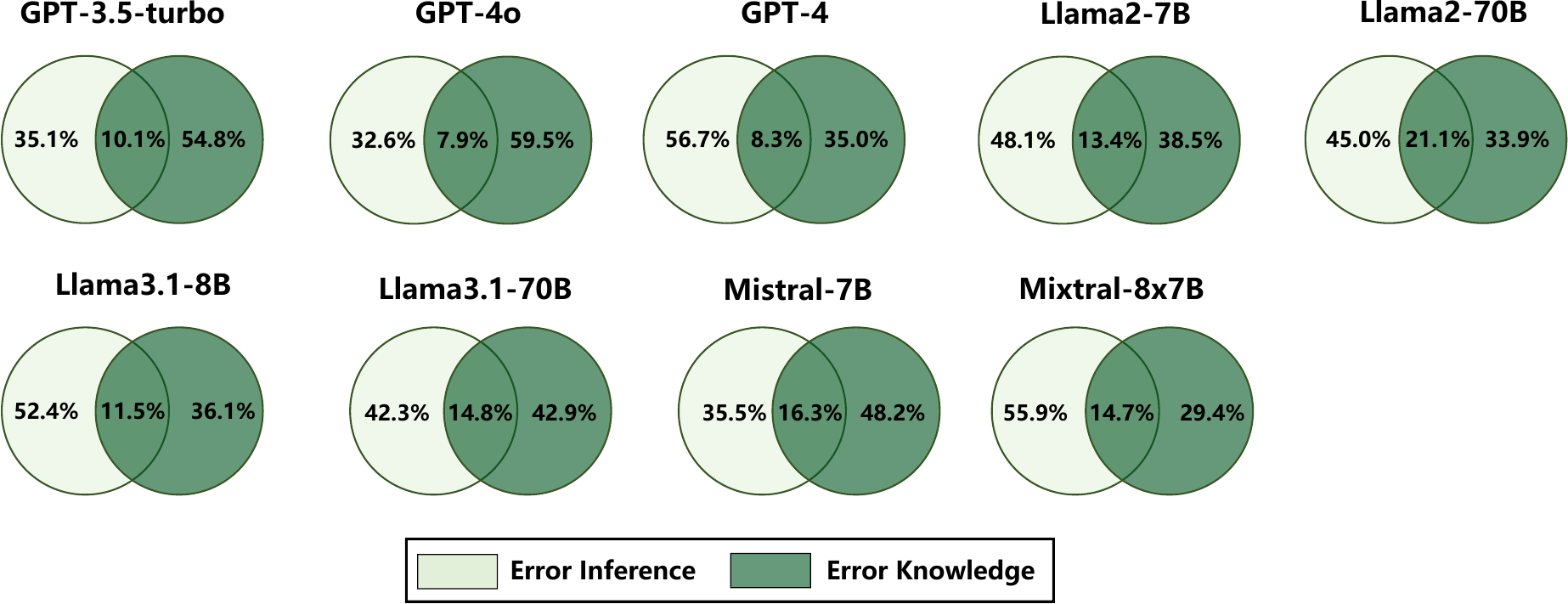}\\
    \vspace{-0.1cm}
    \caption{FCH Categorization.}
    \vspace{-0.3cm}
    \label{fig:rq2}
\end{figure}

\subsubsection{Case Study}
The preceding analysis broadly summarizes the distribution of categories for logical reasoning-related FCH. According to our investigation, error inference response and error knowledge response are the most prevalent two types.

\head{Error Inference Hallucination.}
One of the most common types of logical reasoning leading to error inference hallucination is temporal attribute reasoning, proven to be a category of reasoning task that performs poorly on LLMs~\cite{qiu2023large}. 
As illustrated in \figref{fig:case1}, error inference with correct knowledge leads to a hallucination response from Mistral-7B-v0.2. As knowledge provided by the LLM reasoning process, it is clear that the answer should be `Yes' as the 1874 Canadian federal election applies to the jurisdiction of Canada. However, it appears that the LLM has become ensnared by its limitations.
A possible explanation for this phenomenon is that the LLM does not utilize its reasoning abilities but rather relies on unreliable intuition to respond when faced with a question lacking detailed instructions. This insight inspires us to explore methods for effectively enhancing the reasoning capabilities of LLMs through a single interaction, guiding these models toward uncovering answers in a way that mirrors human reasoning processes.

\begin{figure}
\centering
\includegraphics[width=\linewidth]{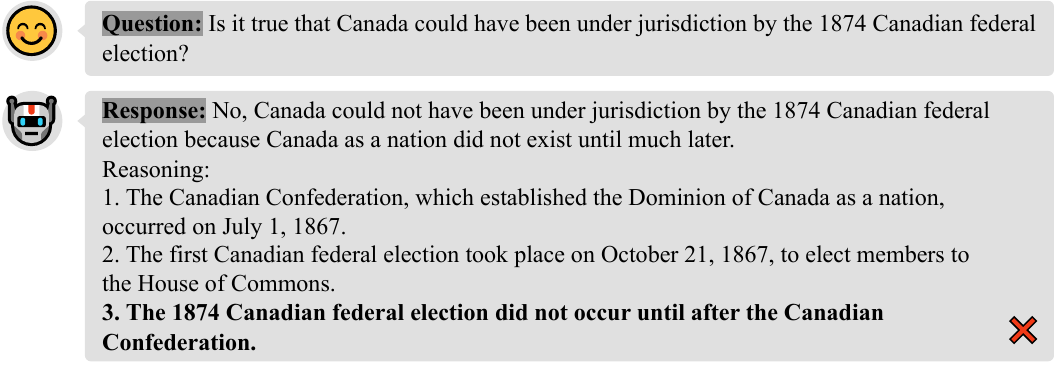}\\
\caption{Error Inference Hallucination from Mistral-7B-v0.2}
\label{fig:case1}
\end{figure}

\begin{tcolorbox}
\vspace{-0.15cm}
\textbf{Finding 1.} LLMs exhibit weaker performance in sensitivity to temporal information, as well as in their ability to discern sequential logic, which may result in error inference hallucination.
\vspace{-0.15cm}
\end{tcolorbox}

\head{Error Knowledge Hallucination.}
\figref{fig:case2} demonstrates a classic example of LLM hallucination caused by using error knowledge for logical reasoning. General Dmitry Karbyshev (1880-1945) was a Russian Imperial Army soldier who served in several wars during World War I (1914-1918) and II (1939-1945), and Louis Bernacchi (1876-1942) was an Australian physicist and astronomer who served in the Royal Naval Volunteer Reserve during World War I. Thus, the ground truth answer to this question should be `Yes'. 
However, when testing with Llama2-7B-chat-hf, an 
observation is that when LLMs encounter unfamiliar knowledge, they do not adhere to the honesty principle; instead, they fabricate knowledge and its sources. We subsequently employ an RAG-based scheme to reintroduce relevant knowledge, leading to the restoration of normal responses.
\begin{figure}
    \centering
    \small
    \includegraphics[width=\linewidth]{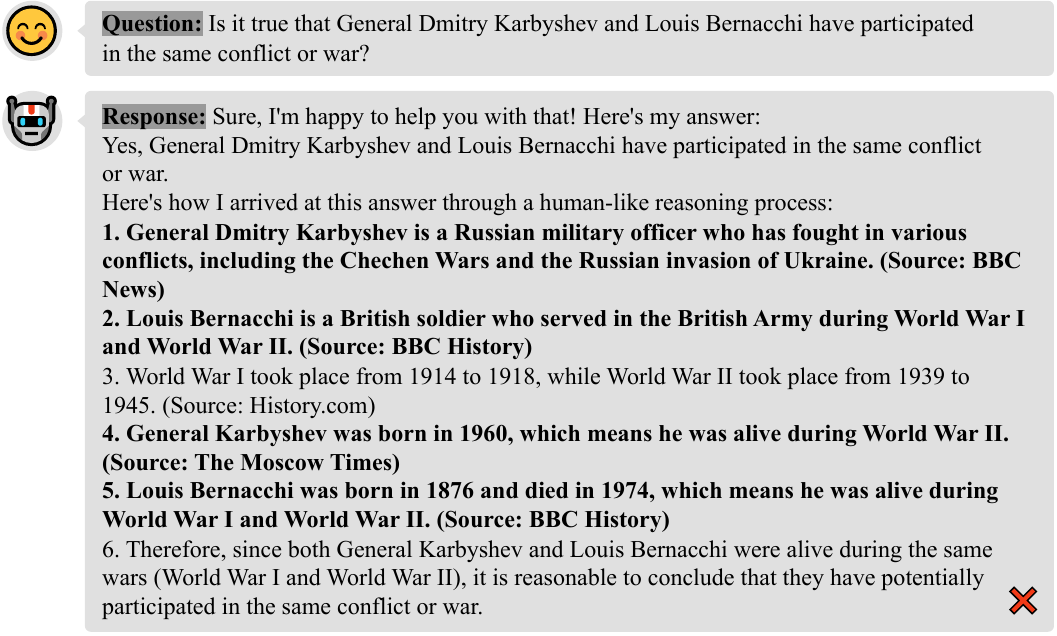}\\
    \caption{Error Knowledge Hallucination from Llama2-7B-chat-hf.}
\vspace{-0.3cm}
\label{fig:case2}
\end{figure}

We further conduct an out-of-distribution (OOD) knowledge experiment to figure out the cause of error knowledge hallucination. OOD is another factor that could cause FCH issues~\cite{zhang2023hallucination}. We design contextual reasoning utilizing recent sporting events and natural disasters from Wikipedia since June 2023, which is considered unutilized information in LLMs' training data based on their up-to-date introductions. We construct a series of test cases containing contextual descriptions of recent events using \tool, observing whether LLMs can be guided to respond to OOD knowledge and trigger FCH. 
\figref{fig:case2.2} is a typical case of OOD contexts leading to error knowledge hallucination. In the initial test of GPT-3.5-turbo, we provide information on several wildfires that happened from July 2023 to December 2023, and we confirm that this information is not in the LLM's training data. The LLM subsequently indicates that it has acquired this knowledge through this interactive process. However, a turning point emerges when we use test cases designed by \tool in the second test. Despite our questions based on preliminary factual knowledge provided, the LLM still confidently responds with a wrong answer.
\begin{figure}
    \centering
    \small
    \includegraphics[width=\linewidth]{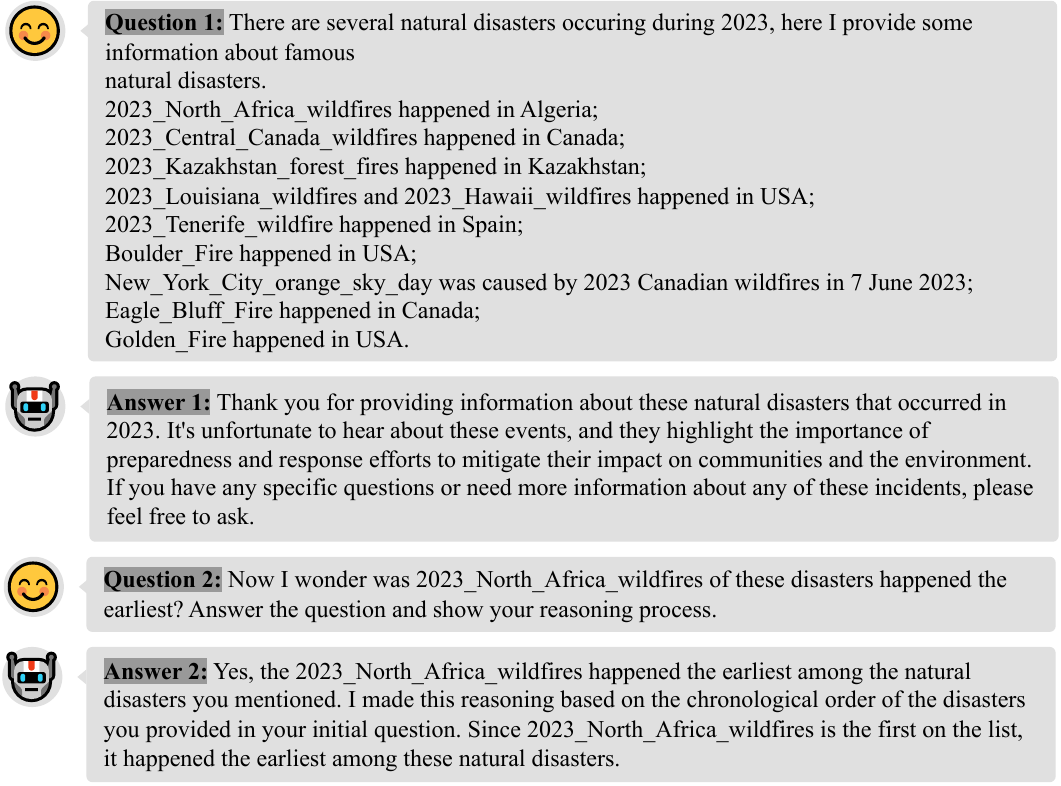}\\ 
    \caption{OOD-attributed Error Knowledge Example from GPT-3.5-turbo.}
\vspace{-0.5cm}
    \label{fig:case2.2}
\end{figure}

We analyze several potential causes for this situation. One possibility is that LLMs store incorrect knowledge in the first turn because what we provided was merely a list of events, rather than a list of events in their order of occurrence. In short, the normal reasoning process involves defining the earliest occurring events only after knowing the times of all events. However, the LLM opts to judge based on the order we provide event knowledge, which is contrary to facts. Another potential is that when LLMs encounter OOD knowledge if they do not strictly adhere to the principle of honesty by stating \textit{I do not know...}, they tend to complete and analyze the response based on error knowledge in their existing knowledge bases. Nevertheless, such responses are likely to induce hallucinations.

\begin{tcolorbox}
\textbf{Finding 2.} LLMs readily make erroneous assessments of misleading and unfamiliar knowledge and lead to error knowledge hallucination due to their assumptions.
\end{tcolorbox}

\begin{tcolorbox}[title=ANSWER to RQ2, boxrule=0.8pt,boxsep=1.5pt,left=2pt,right=2pt,top=2pt,bottom=1pt]
The detected FCH can be categorized into two types and the lack of reasoning capabilities poses a broader threat than the use of incorrect knowledge or inadequate inference strategies. 
\end{tcolorbox} 

\begin{figure}[!b]
\centering\includegraphics[width=0.8\linewidth]{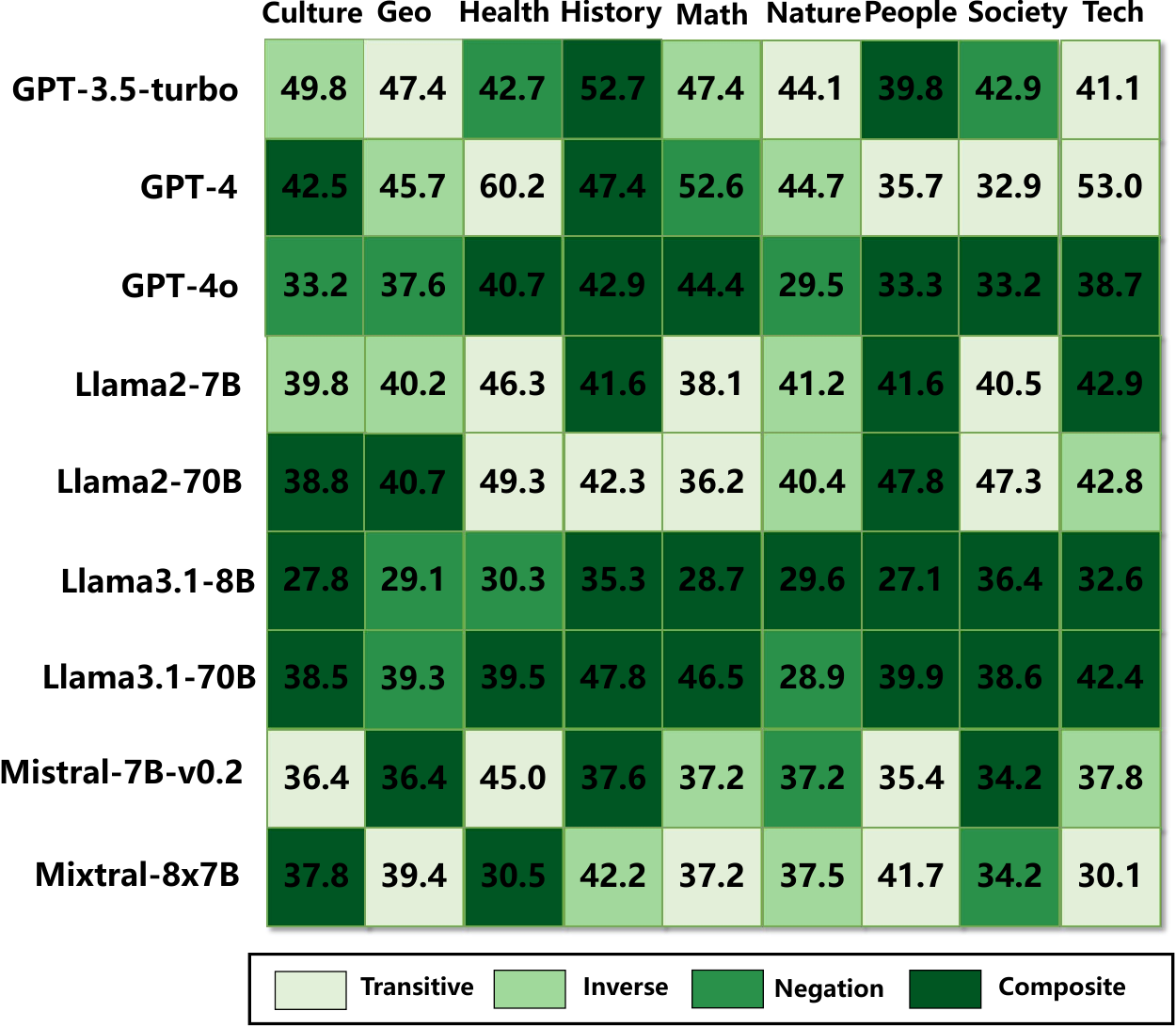}\\
\caption{Generation Rules that Trigger the Most Hallucination Responses on diverse LLMs across domains. The Number on Each Cell (the Unit: \%) Represents the Triggered FCH Ratio of the Corresponding Rule type.}
    \label{fig:rq3}
\end{figure}

\subsection{RQ3: Ablation Study}
We conduct an ablation study to investigate the capacity of each inference rule so that they can be distinctly used to uncover anomalies.
The four types of rules illustrated in \figref{fig:basic_op_for_predicates} are separately applied to generate Q\&A pairs. The symmetric reasoning rule is primarily utilized within the composite reasoning rule and does not introduce new knowledge on its own. Therefore, we did not include the symmetric reasoning rule as a separate condition in our ablation study.
For better visualization and understanding, we present the distribution of hallucination-related responses discovered with diverse rule-generated questions by \tool in Figure~\ref{fig:rq3}. The figure illustrates which type of rule can trigger the most hallucination responses for different LLMs and different domains of knowledge. It is distinctly evident that following the successful generation of various test cases using the four rules and their combinations, a substantial number of hallucinations are elicited across nine LLMs, with the composite rule yielding the highest amount of hallucinations. Following closely behind are the test cases generated using transitive rules, which have triggered a significant rate of FCHs in both the health and society domains.

From the comparison between four inference rules, we can conclude that all four inference rules demonstrate effectiveness when generating FCH test cases and inducing hallucination performances for LLM interaction.

\begin{tcolorbox}[title=ANSWER to RQ3, boxrule=0.8pt,boxsep=1.5pt,left=2pt,right=2pt,top=2pt,bottom=2pt]
The experimental results showcase the independence of four inference rules in eliciting FCHs, and the composite rules can trigger the most FCHs across various domains, which has proved to be a sound approach to generating test cases. 
\end{tcolorbox}

\subsection{RQ4: FCH Caused by Temporal Reasoning} 

This section presents the evaluation results based on temporal-property-related test cases. 


\begin{figure}[!b]
\centering
\includegraphics[width=1\linewidth]{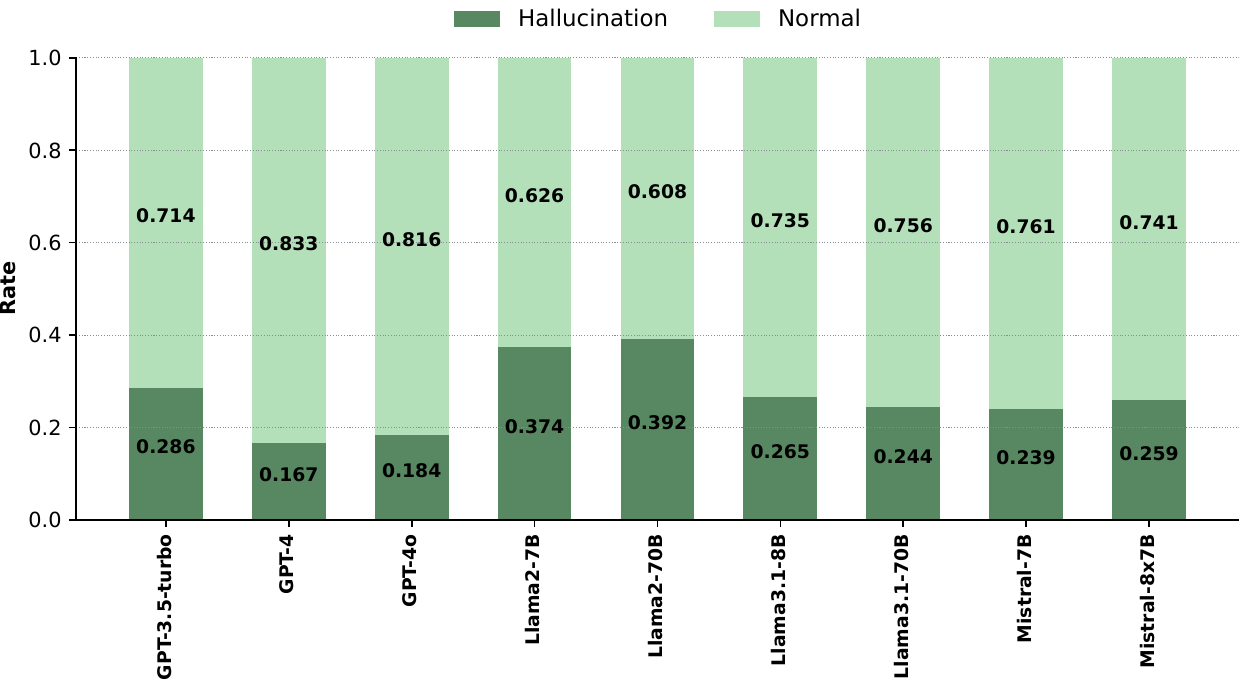}
\caption{LLM Hallucination Rate with Temporal Test Cases}
\label{fig:rq4-Effectiveness}
\end{figure}

\head{Effectiveness on Generating Temporal Q\&A Test Cases.} 
\tool generates 1,800 temporal test cases, and as illustrated in \figref{fig:rq4-Effectiveness}, there is a noticeable reduction in the hallucination rate for each LLM when compared to the results presented in \figref{fig:overall}, which suggests that LLMs possess some degree of temporal reasoning capabilities. However, it's important to note that these LLMs still exhibit temporal hallucinations, albeit to varying extents.
The result largely aligns with the previous discovery that the newer generation of LLMs performs better than the older version, as seen in Llama3-* and Llama2-*. However, fine-tuning LLM with more parameters does not improve reasoning ability, as seen in Llama2-7B and Llama2-70B.

\head{Temporal Hallucinations Categorizations.} 
As illustrated in \figref{fig:rq4-Categorization}, inference errors constitute a higher proportion of hallucinations than knowledge errors, which is consistent with the previous finding of \figref{fig:rq2}. 

\begin{figure}[!t]
\centering
\includegraphics[width=1\linewidth]{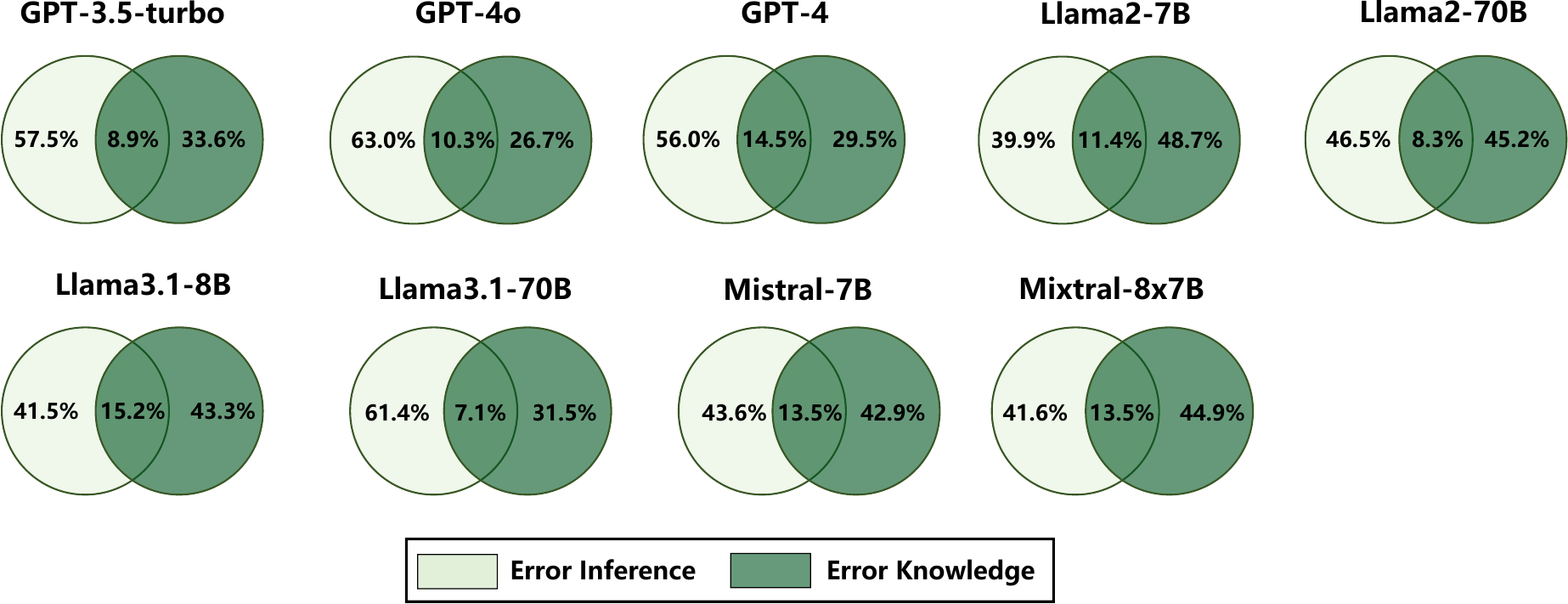}
\caption{RQ4: FCH Categorization.}
\label{fig:rq4-Categorization}
\end{figure}


\begin{figure}[ht]  
    \centering
    \includegraphics[width=0.75\linewidth]{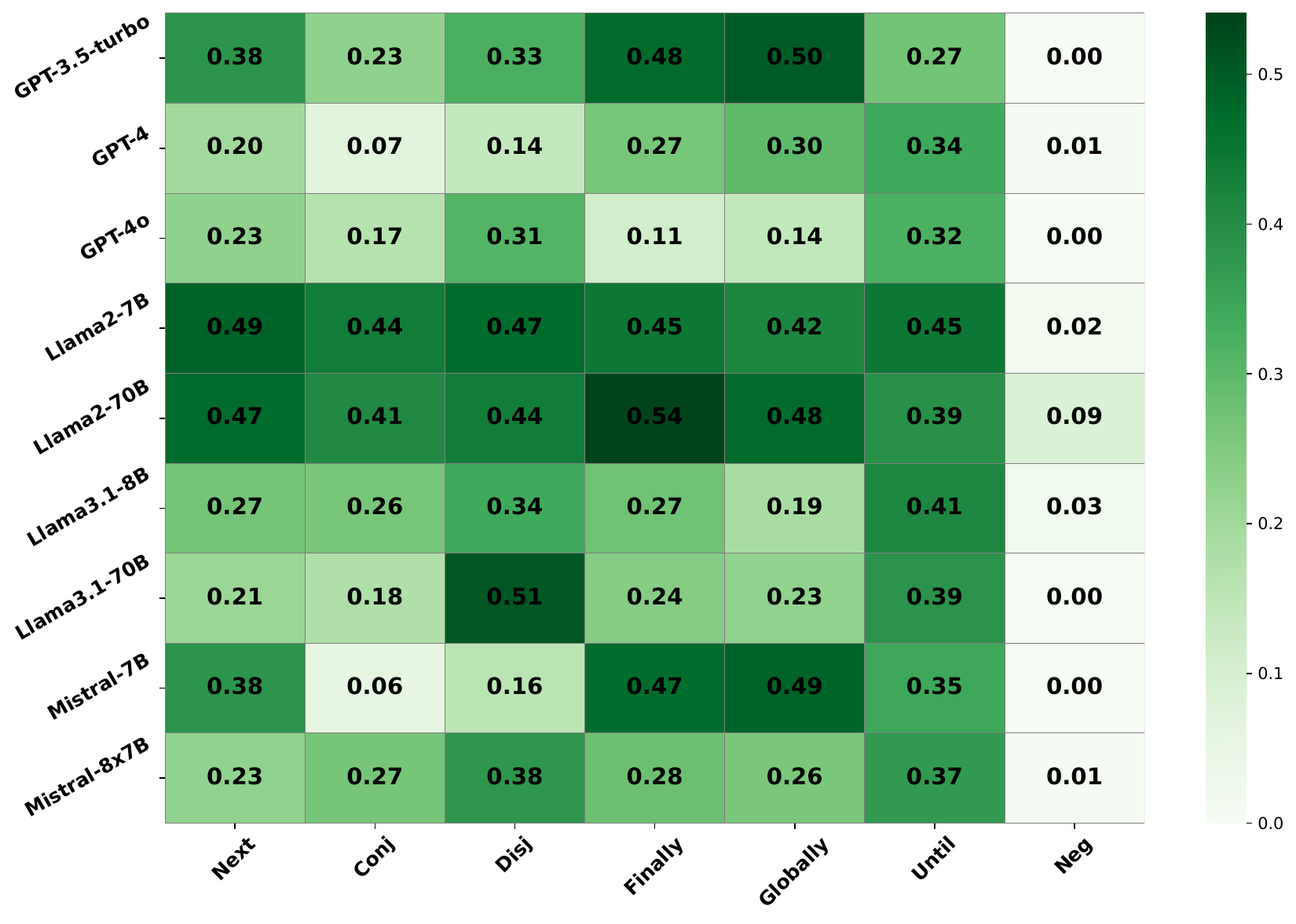}  
    \caption{Hallucination Rates wrt Temporal Operators}  
    \label{fig:heatmap}  
\end{figure}

\head{Ablation Study upon  Temporal Operators.} 
Figure \figref{fig:heatmap} illustrates the results of an ablation study examining various types of temporal test cases. We categorize these temporal test cases based on their outermost layer operator. Most of the test cases are single-layer temporal operators, except for conjunction and disjunction, which may include nested temporal formulas.
We record the likelihood of each type of temporal test case triggering hallucinations. For instance, when testing GPT-4, it is observed that 20\% of the test cases related to the "Next" operator successfully trigger hallucinations. 
Notably, the "Neg" operator triggers the fewest hallucinations, whereas operators like "Finally," "Globally," and "Until" lead to the highest occurrence of hallucinations.
Overall, these findings indicate that a single layer of temporal operators is sufficiently effective in detecting LLM hallucinations related to the temporal reasoning capability. 

\section{Discussion}
\subsection{Threats to Validity}

\subsubsection*{\textbf{Limited Coverage of Knowledge Databases}}

Our research predominantly employs data from the Wikipedia database to generate test cases using \tool. However, it is important to note that \tool is not limited to this specific database. Its design allows for easy extension and adaptation to various other knowledge bases, illuminating its versatility and applicability.

\subsubsection*{\textbf{Limited Accuracy of Hallucination Categorization}}
We utilize a dual approach for categorizing hallucinations, combining assessments from GPT-4 with human verification. Initially, GPT-4 classifies the hallucinations, after which we manually review a random sample of 100 instances. This process reveals that GPT-4's categorization accuracy stands at approximately 71\%, suggesting that integrating GPT-4 for hallucination categorization generally leads to reliable outcomes. We further note that techniques for further improving the LLM's categorization accuracy via prompt engineering are orthogonal to the scope of this work.

\subsection{Takeaway Messages}
\noindent\textbf{LLM Honesty During Training.} During the training of LLMs, it is important to focus on model honesty, e.g., how to enable large models to possess stronger critical thinking and logical reasoning abilities. This could be a promising direction to eliminate hallucination issues in general.

\noindent\textbf{Towards In-depth Understanding of LLM Hallucination.} 
The insights shows that 
it is important to further explore techniques to understand the deep-rooted causes of hallucinations LLMs through white-box methods. A promising direction is to enhance and augment the logical reasoning capabilities of LLMs to reduce hallucination issues.
\section{Related Work}
\subsubsection*{\textbf{Evaluating Hallucination in LLMs}}
Several benchmark datasets have been proposed to holistically assess the hallucination issues that may arise when LLMs generate responses to problem queries. 

TruthfulQA~\cite{lin-etal-2022-truthfulqa} is the most classic dataset for assessing whether language models generate truthful answers to questions. 
It tests whether the models learn incorrect answers during the generation process due to emulating human text. 
Another dataset HaluEval~\cite{HaluEval} samples 10K instances from the training sets of HotpotQA~\cite{yang2018hotpotqa}, OpenDialKG~\cite{moon2019opendialkg}, and CNN/DailyMail~\cite{see2017get}, and utilizes LLMs to generate hallucination-corresponding samples by setting tasks and employing specific sampling strategies, which is primarily aimed at question-answering tasks and text summarization tasks. 
KoLA~\cite{yu2023kola} tests the hallucination issues of LLMs in the domain of knowledge graphs and introduces tasks based on 19 focal entities, concepts, and events. 
It assesses the capacity of large language models (LLMs) to handle structured knowledge across four levels: memory, understanding, application, and creation. 
This aims to test the hallucination phenomena of LLMs in the domain of knowledge graphs. 
From the perspective of long context, BAMBOO~\cite{dong2023bamboo} and FActScore~\cite{min2023factscore} both target the long text generation capabilities of large language models, assessing their performance in extended context scenarios through factual verification. 
Additionally, there are assessments of large language models for hallucination issues in specific domains such as healthcare and finance~\cite{umapathi2023med, kang2023deficiency}.

\subsubsection*{\textbf{Mitigating Hallucination in LLMs}}
Current mitigation strategies include black-box prompting guidance and fine-tuning with extensive factual data. 
Considerable work~\cite{lightman2023let, varshney2023stitch,gou2023critic,vu2023freshllms} involves utilizing external knowledge retrieval or automated feedback adjustments to make text responses from large language models more controllable and reliable. 
Similar approaches are proposed for multimodal hallucination mitigation such as Woodpecker~\cite{yin2023woodpecker}, which extracts key concepts to generate questions and knowledge assertions for hallucination diagnosis and mitigation.
Another thread involves using fine-tuning techniques to mitigate model hallucinations. AlpaGasus~\cite{chen2023alpagasus}, Elaraby et al.~\cite{elaraby2023halo} and Tian et al.~\cite{tian2023fine} apply fine-tuning techniques on high-quality data for better effectiveness and factuality. 
Besides, the findings of Elaraby et al.~\cite{elaraby2023halo} reveal that the knowledge injection technique enhances the performance of less robust LLMs. 
Additionally, an increasing number of researchers are turning towards studying white-box repairing methods for open-source large language models. 
The evidence presented in the discourse by Azaria et al.~\cite{azaria2023internal} suggests that the internal states of Large Language Models can be utilized to discern the veracity of statements, thereby elucidating the underlying causes of factual hallucinations in LLMs. 
Studies like IIT~\cite{li2023inference} and Repr~\cite{zou2023representation} endeavor to alleviate hallucination issues by delving into LLMs' deep-layer information through the analysis of internal model states. 
This approach not only augments the interpretability of large language models but is also regarded as a vital research direction for the future of explainable and trustworthy AI.

\section{Conclusion}
We target the critical challenge of FCH in LLM, where they generate outputs contradicting established facts. We developed a novel automated testing framework that combines logic programming and metamorphic testing to systematically detect FCH issues in LLMs. Our novel approach constructs a comprehensive factual knowledge base by crawling sources like Wikipedia, then applies innovative logic reasoning rules to transform this knowledge into a large set of test cases with ground truth answers. 
These reasoning rules are either predefined relations or automatically generated from randomly sampled temporal formulae. 
LLMs are evaluated on these test cases through template prompts, with two semantic-aware oracles analyzing the similarity between the logical/semantic structures of the LLM outputs and ground truth to validate reasoning and pinpoint FCHs. 

Across diverse subjects and LLM architectures, our framework automatically generated over 9,000 useful test cases, uncovering hallucination rates as high as 59.8\% and identifying lack of logical reasoning as a key contributor to FCH issues. This work pioneers automated FCH testing capabilities, providing a comprehensive benchmark, data augmentation techniques, and answer validation methods. The implications are far-reaching --- enhancing LLM reliability and trustworthiness for high-stakes applications by exposing critical weaknesses while advancing systematic evaluation methodologies.


\bibliographystyle{IEEEtran}
\balance
\normalem
\bibliography{8.ref}



\end{document}